\documentclass[twoside,11pt]{article}

\usepackage{blindtext}

\usepackage[preprint]{jmlr2e}

\usepackage{lastpage}
\jmlrheading{26}{2025}{1-\pageref{LastPage}}{1/21; Revised 5/22}{9/22}{21-0000}{Arya Akhavan, Karim Lounici, Massimiliano Pontil, and Alexandre B. Tsybakov}

\ShortHeadings{A conversion theorem and minimax optimality for continuum
contextual bandits}{Akhavan, Lounici, Pontil, Tsybakov}
\firstpageno{1}
\PassOptionsToPackage{numbers, sort&compress}{natbib}
\usepackage{natbib}
\usepackage{ulem}
\usepackage{lmodern}

\usepackage[utf8]{inputenc} % allow utf-8 input
\usepackage[T1]{fontenc}    % use 8-bit T1 fonts
\usepackage{hyperref}       % hyperlinks
\usepackage{url}            % simple URL typesetting
\usepackage{booktabs}       % professional-quality tables
\usepackage{amsfonts}       % blackboard math symbols
\usepackage{nicefrac}       % compact symbols for 1/2, etc.
\usepackage{microtype}      % microtypography
\usepackage{xcolor}         % colors

\usepackage{times}
\usepackage{amsmath, nccmath, amssymb}

\usepackage{geometry}
\usepackage{comment}
\usepackage{enumitem}
\usepackage{color}
\usepackage{graphicx}
\usepackage{bbold}
\usepackage{xfrac}
\usepackage{thmtools}
\usepackage[ruled]{algorithm2e}
\usepackage{algorithmic}
\usepackage{mathtools}
\usepackage{mathrsfs}  
\usepackage{caption}

\usepackage{notation}
\newcommand{\Lc}{L}
\newcommand{\Lx}{\beta}
\newcommand{\dc}{p}
\newcommand{\dx}{d}
\newcommand{\dxmu}{\rho}
\newcommand{\numbin}{K}
\newcommand{\numhypo}{\{-1,1\}^{\dx-1}}

\usepackage{tikz}
\begin{document}

\title{{A conversion theorem and minimax optimality for continuum contextual bandits}}
 %: improved regret bounds in the convex case}
 
%  \title{A new l1-randomized gradient estimator for zero-order optimization: improved regret bounds in the convex case}

%\title{Improving regret bounds in zero-order optimization via l1-randomized gradient estimator for online convex} %: a new gradient estimator and improved regret bounds}

%\title{Exploiting L1-randomization for online convex  zero-order optimization}%: a new gradient estimator and improved regret bounds}

% \title{A new gradient estimator two point feedback and improved bound for convex online zero-order optimization}

% \title{Convex online zero-order optimization with two point feedback: new gradient estimator}

% The \author macro works with any number of authors. There are two commands
% used to separate the names and addresses of multiple authors: \And and \AND.
%
% Using \And between authors leaves it to LaTeX to determine where to break the
% lines. Using \AND forces a line break at that point. So, if LaTeX puts 3 of 4
% authors names on the first line, and the last on the second line, try using
% \AND instead of \And before the third author name.

\author{\name Arya Akhavan\email arya.akhavan@stats.ox.ac.uk \\
      \addr University of Oxford\\
      CMAP, Ecole Polytechnique, IP Paris
      \AND
      \name Karim Lounici \email karim.lounici@polytechnique.edu \\
      \addr
      CMAP, Ecole Polytechnique, IP Paris
      \AND
      \name Massimiliano Pontil \email massimiliano.pontil@iit.it \\
      \addr 
    CSML, Istituto Italiano di Tecnologia\\
      University College London
      \AND
      \name Alexandre B. Tsybakov \email alexandre.tsybakov@ensae.fr \\
      \addr CREST, ENSAE, IP Paris
      }
      
\editor{}

\maketitle
\begin{abstract}
We study the contextual continuum bandits problem, where the learner sequentially receives a side information vector and has to choose an action in a convex set, minimizing a function associated to the context. 
The goal is to minimize all the underlying functions for the received contexts, leading to the {contextual} notion of regret, which is stronger than the standard static regret.
Assuming that the objective functions are $\gamma$-H\"older with respect to the  contexts, $0<\gamma\le 1$, we demonstrate that any algorithm achieving a sub-linear static regret can be extended to achieve a sub-linear {contextual} regret. {We prove a
static-to-contextual regret conversion theorem that provides an upper bound for the contextual regret
of the output algorithm as a function of the static regret of the input algorithm. We further study implications of this general result for three fundamental cases of dependency of the objective function on the action variable: (a) Lipschitz bandits, (b) convex bandits, (c) strongly convex and smooth bandits. For Lipschitz bandits and $\gamma=1$, combining our results with the lower bound of \cite{slivkins14a}, we prove that the minimax optimal contextual regret for the noise-free adversarial setting scales as $T^{(\dc+\dx+1)/(\dc+\dx+2)}$ up to logarithmic factors, where $p$ and $d$ are the dimensions of the context and of the action spaces, respectively, and $T$ is the number of queries. Then we prove that, when the evaluations are noisy, the rate of contextual regret in $T$  for convex bandits is the same as for strongly convex and smooth bandits and scales as $T^{(\dc+\gamma)/(\dc+2\gamma)}$ up to logarithmic factors.} 

Lastly, we present a minimax lower bound, implying two key facts. First, obtaining a sub-linear contextual regret may be impossible over functions that are not continuous with respect to the context. Second, for convex bandits and strongly convex and smooth bandits, the algorithms that we propose achieve, up to a logarithmic factor, the minimax optimal rate of {contextual} regret as a function of the number of queries.

\end{abstract}

\section{Introduction}
The continuum bandit problem presents a sequential decision-making challenge where, at each round, an action is taken as a continuous variable followed by suffering a corresponding loss. Various variants of this problem have been explored in the literature. The most extensively studied framework involves the learner receiving a function $f_t: \com\rightarrow \mathbb{R}$ (random or non-random), at each round $t$, where $\com\subseteq \mathbb{R}^{d}$ is a convex body (compact and convex set with a non-empty interior). Based on $\{(y_k,\bz_{k})\}_{k=1}^{t-1}$, the learner selects $\bz_t \in \com$, suffers loss $f_t(\bz_t)$, and receives the noisy feedback
$$y_t = f_t(\bz_t)+\xi_t,$$ 
where $\xi_t$ denotes the noise variable. The standard goal of the learner is to minimize the static regret
\begin{align}\label{eq:reg}
   R_T:= \Exp\Big[\sum_{t=1}^{T}f_t(\bz_t) - \min_{\bz\in\com}\sum_{t=1}^{T}f_t(\bz)\Big].
\end{align}
In the present paper, we assume that the learner has access to side information about the loss function, namely a context at each round. Leveraging this additional information, the learner is expected to minimize another type of regret referred to as the contextual regret. We consider the problem of contextual continuum bandits stated as follows.

\textbf{Contextual continuum bandits.} Let $\com\subseteq \mathbb{R}^{\dx}$ be a convex body. Let $f:\mathbb{R}^{\dx}\times [0,1]^{\dc}\to\mathbb{R}$ be an unknown function. At each round $t$:
\begin{itemize}
    \item A context $\bc_t\in [0,1]^{\dc}$ is revealed by the adversary.
    \item Based on the past values $\{ (y_k,\bz_k,\bc_k) \}_{k=1}^{t-1}$, the learner chooses a query point $\bz_t\in\com$ and gets a noisy evaluation of $f(\bz_t,\bc_t)$:
    \begin{align}\label{eq:observe}
        y_t = f(\bz_t,\bc_t) + \xi_t,
    \end{align}
    where $\xi_t$ is a scalar noise variable.
\end{itemize}
The learner's objective in this framework is to minimize the \textbf{contextual regret}, defined as:
\begin{align}\label{eq:contxtreg}
  R_T^{\mathrm{cntx}}:= \Exp\Big[\sum_{t=1}^{T}f(\bz_t,\bc_t) - \sum_{t=1}^{T}\min_{\bz\in\com}f(\bz,\bc_t)\Bigg].
\end{align} 

Some remarks about this setting are in order. As a general framework for the problem, we assume that the contexts are drawn sequentially from the hypercube $[0,1]^{\dc}$. This presentation is chosen to simplify the notation. However, all the upper bounds on the regret throughout the paper hold if the contexts are drawn from a non-empty compact subset of $[0,1]^{\dc}$. Moreover, at first sight, by introducing the notation \( f_t(\bz) = f(\bz, \bc_t) \), the contextual continuum bandit problem as stated above might appear as a special case of the {\it non-contextual} continuum bandit problem with dynamic regret (CBDR). In the CBDR setting, the learner selects the query point \( \bz_t \) at time \( t \) based on past observations \( \{y_k, \bz_k\}_{k=0}^{t-1} \) to minimize the dynamic regret: 
\begin{align}\label{eq:regModifid}
  {R_T^{\mathrm{dyn}} :=}  \Exp\Big[\sum_{t=1}^{T}f_t(\bz_t) - \sum_{t=1}^{T}\min_{\bz\in\com}f_t(\bz)\Big].
\end{align}
 With this notation, we have \( {R_T^{\mathrm{cntx}} = R_T^{\mathrm{dyn}}} \). However, the two settings are not equivalent. Indeed, the contextual setting that we consider is ``statistically easier" than CBDR since partial information about \( f_t \) in the form of the context \( \bc_t \) is revealed to the learner before the action is chosen. This suggests that a better behavior of the regret is potentially achievable.

   Another performance measure for continuum contextual bandits called the $L$-regret was introduced by \cite{hazan2007online}:
 \begin{align}\label{eq:L-regret}
  R_T^{L}:= \mathbf{E}\left[\sum_{t=1}^{T} f_t(\bz_t) - \min_{\bz(\cdot)\in X_L}\sum_{t=1}^{T}f_t(\bz(\bc_t))\right].
\end{align} 
Here, the minimum is taken over the class $X_L$ of all actions satisfying the Lipschitz condition with respect to the context, with given Lipschitz constant $L$. For $L=0$ the $L$-regret is equal to the static regret $R_T$.

 Considering \( f_t(\bz) = f(\bz, \bc_t) \) we have that $R_T^{\mathrm{cntx}} =  R_T^{\mathrm{dyn}} \ge R_T^{L} \geq R_T$. The setting with $L$-regret, where $L>0$, is more challenging than the static regret setting as constant actions trivially satisfy the Lipschitz condition. In turn, controlling the contextual regret $R_T^{\mathrm{cntx}}$ is more challenging than the $L$-regret since it evaluates the risk against all possible choices of actions and not only against the Lipschitz actions. 
 %that the contextual setting is more challenging than the static setting, as the dynamic regret always exceeds the static regret, \( \KL{R_T^{\mathrm{cntx}} =  R_T^{\mathrm{dyn}} \geq R_T} \). 
 By definition, small static regret indicates that the policy is at least as effective as a single constant action performing well on the average of all observed functions during learning. In contrast, small dynamic regret signifies that the policy performs well for each function individually. For non-contextual bandits, the CBDR problem has been recently investigated  by \cite{yang2016tracking, chen2018bandit, zhao2021bandit,WeiLuo2021non-stationary,wang2023adaptivity} under restrictions on various characteristics of the evolution of \( f_t \)'s over time, which are not interpretable in terms of the smoothness assumptions on $f(\cdot,\cdot)$  used in continuum contextual settings. %In this paper, we address the central question: Under what conditions can sub-linear contextual regret be achieved?

In Section \ref{sec:lower}, we demonstrate that, regardless of the behavior of the function $f(\cdot,\bc)$, if the function $f(\bx,\cdot)$ lacks continuity, obtaining a sub-linear contextual regret may be impossible. Hence, we focus on the setting where $f(\bx, \cdot)$ is a continuous function for any $\bx\in\mathbb{R}^{d}$. Specifically, in the sequel we adopt  the following slightly stronger assumption. 
\begin{assumption}\label{ass:Lip}
 For $(\Lc,\gamma)\in[0,\infty)\times (0,1]$ we assume that $f\in \mathcal{F_{\gamma}}(\Lc)$, where $\mathcal{F_{\gamma}}(\Lc)$ denotes the class of all functions $f:\mathbb{R}^{\dx}\times [0,1]^{\dc}\to\mathbb{R}$ such that 
\begin{align}\label{eq:goodalgo}
    |f(\bx,\bc) - f(\bx,\bc')|\leq \Lc\norm{\bc-\bc'}^{\gamma},\quad\text{for all}\quad \bx\in\com, \,\,\bc,\bc'\in[0,1]^{\dc},
    \end{align}
    where $\norm{\cdot}$ is the Euclidean norm.
\end{assumption}

%\st{In the setting of non-bandit online zero- and first-order optimization, several authors have extensively explored the minimization of dynamic regret eq:regModifid, see e.g., mokhtari2016online, zinkevich2003online, besbes2015non, hall2015online, jadbabaie2015online,zhao2021improved. To achieve sub-linear dynamic regret, these works introduce complexity assumptions on the received functions. Such assumptions characterize similarity between two consecutive functions helping the learner to leverage previously received information for subsequent inference.}

In this paper, we develop a meta-algorithm approach to continuum contextual bandit problem. In short it can be described as follows. Let $\mathcal{F}'$ be a class of real-valued functions on $\mathbb{R}^d$. Assume that we are given an input algorithm $\pi$
achieving a sub-linear
{\it static regret} in adversarial bandit setting over $T$ runs for the sequence of functions $\{f_t(\cdot)\}_{t=1}^{T}$, where each $f_t\in \mathcal{F}'$. Then, we associate to it an output contextual bandit algorithm with sub-linear
{\it contextual regret} for any loss $f(\cdot,\cdot)$ belonging to the class
$\mathcal{G}_{\gamma}(\mathcal{F}'):=\{f\in \mathcal{F_{\gamma}}(\Lc): f(\cdot,\bc)\in \mathcal{F}' \ \text{for all} \, \bc\in [0,1]^{p}\}$. We prove a static-to-contextual regret conversion theorem that provides an upper bound for
the contextual regret of the output algorithm as a function of the static regret of the input algorithm.
Our construction of the output algorithm consists in discretizing the space of contexts into cells and running an algorithm \( \pi \) separately within each cell. Similar constructions can be traced back to \cite{hazan2007online}, where the proposed method was not a meta-algorithm valid for any $\pi$ and $\mathcal{F}'$ but a single contextual bandit algorithm with a specific input algorithm \( \pi \) (projected gradient) and a specific class \( \mathcal{F}' \) (convex and Lipschitz functions). The analysis given in \cite{hazan2007online} focuses on the \( L \)-regret and does not provide a control of the contextual regret. The subsequent work \cite{cesa2017algorithmic}  adopts a similar discretization strategy, with similar limitations. It derives a bound on the \( L \)-regret when the input algorithm is chosen as Exp3 and the class \( \mathcal{F}' \) is the set of all 1-Lipschitz functions.   

Another approach to static-to-contextual regret conversion is developed in \cite{slivkins14a}. It relies
on zooming with dyadic splits and is valid, again, for a specific input class \( \mathcal{F}' \), namely, a subset of Lipschitz functions, and a specific class of input algorithms $\pi$, namely, those achieving the optimal rate on this particular class. The class of functions $f(\cdot,\cdot)$ considered in \cite{slivkins14a} is a subclass of $\mathcal{G}_{1}(\mathcal{F}')$ satisfying, along with the Lipschitz condition,  the additional restrictive condition that for each \( \bc \), there exists \( \bz^*(\bc) = \arg\min_{\bz \in \com} f(\bz, \bc) \) such that \( \bc \mapsto f(\bz^*(\bc), \bc) \) is Lipschitz. The algorithm of \cite{slivkins14a} requires exact knowledge of the Lipschitz constant, and the resulting bound on the contextual regret is exponential in the dimension of the context.

\vspace{2mm}

The contributions of the present paper can be summarized as follows.
\begin{itemize} 
    \item[(i)] In Section \ref{sec:general}, we propose a meta-algorithm that converts any input non-contextual bandit algorithm into an output contextual bandit algorithm. We prove a general conversion theorem that derives the contextual regret of the output algorithm from the static regret of the input algorithm (Theorem \ref{thm:general}). 
    \item[(ii)] We apply our static-to-contextual conversion to handle the following two fundamental settings. 
    \begin{itemize}
        \item[(ii-a)] First, we consider the setting, where \( f(\cdot, \cdot) \) is Lipschitz in both arguments, that is, $\gamma=1$ and $\mathcal{F}'$ is the class of Lipschitz functions. In this setting, when the noises $\xi_i$ are zero, we show that our conversion approach leads to an output algorithm with a rate optimal contextual regret improving upon \cite{slivkins14a} in several respects. Unlike \cite{slivkins14a}, we 
    do not require the above restrictive additional condition on $f$ nor the knowledge of Lipschitz constant, and our bound scales polynomially in the dimension of the context. 
    \item[(ii-b)] Second, we consider the setting where $\mathcal{F}'$ is the class of all convex functions and the noise in observations \eqref{eq:observe} is sub-Gaussian. In this case, by applying our conversion theorem together with the static regret bound from \cite{fokkema2024online} we prove that the contextual regret of the proposed output algorithm achieves the minimax optimal rate in \( T \) up to a logarithmic factor.
    \end{itemize}
     \item[(iii)] In Section \ref{sec:strong}, we apply our conversion theorem in the setting where  
$\mathcal{F}'$ is the class of strongly convex and smooth functions, and the noise is sub-Gaussian. We first construct a bandit algorithm, see Algorithm \ref{algoBCO}, which is an extension of the Bandit Convex Optimization (BCO) algorithm from \cite{hazan2014bandit} to the problem with noisy queries. We prove a bound on its static regret. Then, combining it with the conversion theorem of Section \ref{sec:general}, we derive a  bound on the contextual regret of the proposed output algorithm.
\item[(iv)] In Section \ref{sec:lower}, we establish a minimax lower bound leading to several useful corollaries. First, it implies that  the contextual regret bounds from items (ii-b) and (iii) are minimax optimal in $T$ up to logarithmic factors. Second, it shows that if 
$f(\bx,\cdot)$ lacks continuity, obtaining a sub-linear contextual regret may be impossible. 
\end{itemize}
%Specifically, consider functions \( f_t \) represented as \( f_t(\bx) = f(\bx, \bc_t) \), where \( f: \mathbb{R}^{\dx} \times [0,1]^{\dc} \to \mathbb{R} \). We demonstrate that if an algorithm achieves sub-linear static regret \eqref{eq:reg} for all \( f: \mathbb{R}^{\dx} \times [0,1]^{\dc} \to \mathbb{R} \) within a class \( \mathcal{F} \), then our meta-algorithm converts it into an output algorithm that attains sub-linear dynamic regret for functions in \( \mathcal{F} \cap \mathcal{F}_{\gamma}(\Lc) \).  

\section{Related work}

%The material presented below can be approached from three interconnected points of view: (i) the setting of contextual bandits, which is the main focus of this paper; (ii) the study of dynamic regret and the conditions under which sub-linear dynamic regret is achievable; and (iii) the convex bandits setting with strongly convex objective functions (see Sections \ref{sec:strong} and \ref{sec:lower}).

 %problem is a variant of bandit problem, where at each round the learner is equipped with an additional side information on the subsequent objective function. 
 Contextual bandit problem has been introduced by \cite{woodroofe1979one} in the framework of multi-armed bandits, that is, when the space of actions is finite. The literature on the contextual bandit problem with finite number of arms is now very rich, see, e.g., \cite{goldenshluger2009woodroofe,wang2005bandit, bastani2020online, perchet2013multi, rigollet2010nonparametric, pandey2007bandits, langford2007epoch, auer2002finite, kakade2008efficient, yang2002randomized,Foster2023FoundationsOR,CellaLPP23,CellaLPP2024} and the references therein. In this literature, it is not uncommon to assume that
the contexts are i.i.d. rather than adversarial. 

There was also a considerable work on the problem of contextual bandits with continuous action space: \cite{hazan2007online,slivkins14a,kleinberg2008multi, cesa2017algorithmic,chen2018nonparametric,li2019dimension}.
These papers deal with adversarial contexts and, with the exception of \cite{li2019dimension}, consider the noiseless setting ($\xi_i\equiv 0$). The analysis in \cite{slivkins14a} addresses the case where the context and the action vary in metric spaces with covering dimensions $p$ and $d$, respectively, under the assumption that $f(\cdot,\cdot)$ is Lipschitz in both arguments and satisfies an additional strong restriction (see the Introduction). The results in \cite{slivkins14a} provide a lower bound of the order $O(T^{\frac{p+d+1}{p+d+2}})$ on the contextual regret and an algorithm based on zooming techniques that attains this rate up to a logarithmic factor. Bounds for the $L$-regret for continuum contextual bandit algorithms are developed in \cite{hazan2007online,cesa2017algorithmic}. The paper \cite{chen2018nonparametric} deals with the setting where $f(\cdot,\cdot)$ is Lipschitz in both arguments, the actions belong to $[0,1]$, the function $f(\cdot,\bc)$ has a unique minimizer for any context $\bc$ and exhibits approximately quadratic behavior around the minimizer. Under these assumptions,  \cite{chen2018nonparametric} propose an algorithm with $O(T^{\frac{p+2}{p+4}})$ contextual regret. 
%In a framework close to the one that we consider in Section \ref{sec:strong}, the paper %and \ref{sec:lower}, 
In \cite{li2019dimension}, the authors addressed the problem of contextual bandits when the functions are strongly convex and smooth and derived a bound for the contextual regret that scales as $O(T^{\frac{p+1}{p+2}})$ up to logarithmic factors. To get this result, \cite{li2019dimension} additionally assumed that, for each $\bc\in[0,1]^{\dc}$, the global minimizer $\argmin_{\bx\in\mathbb{R}^{\dx}}f(\bx,\bc)$ is an interior point of the constraint set $\com$.
 While this assumption might be hard to check in practice, it does simplify the analysis considerably. In the present paper, we drop this assumption and obtain in Section \ref{sec:strong} an upper bound that improves upon the result of
\cite{li2019dimension}. Furthermore, in Section \ref{sec:lower} we present a matching lower bound demonstrating that the algorithm that we propose is minimax optimal to within a logarithmic factor.

 %\st{THESE ARE NOT CONTEXTUAL: Other studies, such as \cite{kleinberg2004nearly,kleinberg2008multi,bubeck2008online, reduce the continuum contextual problem to multi-armed bandits with a finite number of arms, often relying on zooming techniques with dyadic splits. While these methods achieve sub-linear regret bounds in $ T $, they suffer from exponential dependence on the context dimension.}

% The same regret bound can be achieved by the uniform partition and a bandit-with-expert-advice algorithm
%such as EXP4 [Auer et al., 2002b] or NEXP [McMahan and Streeter, 2009]. The uniform partition
%is used to define an expert whose advise is simply an arbitrary arm for each set of the
%partition. Slivkins [2014] proposes an adaptive zooming algorithm so that frequently occurring
%contexts and high-paying arms structure can be used to improve practical performance. There
%other versions of contextual bandit problems, such as linear bandits [Auer, 2002, Dani et al., 2008,
%Rusmevichientong and Tsitsiklis, 2010, Abbasi-Yadkori et al., 2011], contextual bandits with policy
%sets [Auer et al., 2002b, Langford and Zhang, 2008, Agarwal et al., 2012, Dudik et al., 2011].

In all the papers on continuum contextual bandits cited above, the contexts are adversarial, so the contextual regret is a variant of the dynamic regret as discussed after \eqref{eq:contxtreg}. %However, it is quite common in the literature to assume that the contexts are i.i.d. rather than adversarial.
The dynamic regret was also in the focus of studies in various settings of online convex optimization. For zero-order optimization (aka gradient-free optimization), the dynamic regret has been studied in \cite{zinkevich2003online,yang2016tracking, zhang2018dynamic,chen2018bandit,zhao2021bandit}. Among these papers, \cite{yang2016tracking, chen2018bandit, zhao2021bandit} have addressed the bandit framework, in which the learner is not allowed to query function evaluations outside of the constraint set $\Theta$.

In the framework of adversarial multi-armed contextual bandits with finite number of arms, without imposing any regularity assumptions on the context such as our Assumption \ref{ass:Lip}, \cite{chen2019new,luo2018efficient,auer2019adswitch} achieved sub-linear regret by assuming that the functions
$f_t$	either change infrequently or vary negligibly over time. %Under these conditions, \cite{chen2019new,luo2018efficient,auer2019adswitch} achieved sub-linear regret without requiring regularity properties with respect to the context but rather relying on the assumption that $f_t$'s remain close across time steps.

%Regarding simply convex functions, the first studies date back to  \cite{flaxman2004online,kleinberg2004nearly}, where the authors addressed the setting with convex and Lipschitz objectives and proposed algorithms that achieve a regret of the order $T^{3/4}$ up to a factor depending on $d$. These works assume that the observed function values are noise-free. %Restricting the queries to $\com$ (the genuine bandit setting) appears to be challenging. 
%Successive improvements \cite{bubeck2017kernel,hazan2016optimal,belloni2015escaping,lattimore2020improved} showed that a regret of the order $\sqrt{T}$ (up to a factor depending on $d$) can be attained in the noise-free setting, the sharpest known rate being $\dx^{2.5}\sqrt{T}$  \cite{lattimore2020improved}. Noisy variant of the problem was studied only when the objective function does not vary over time and the noise is sub-Gaussian. This led to algorithms with a regret of the order $\dx^{16}\sqrt{T}$~\cite{agarwal2013stochastic} and $\dx^{4.5}\sqrt{T}$ \cite{lattimore2021improved}. For a detailed survey on convex bandits, we refer the reader to \cite[Sec. 2]{lattimore2024bandit}.

\section{A general conversion theorem}\label{sec:general}

In this section, we propose a meta-algorithm that to any input algorithm achieving a sub-linear static regret in
adversarial bandit setting associates a contextual bandit algorithm with sub-linear contextual regret.
We establish a conversion theorem (Theorem \ref{thm:general} below) that provides an upper bound for the contextual
regret of the output algorithm as a function of the static regret of the input algorithm.

Here and in the sequel, unless other assumption is explicitly stated, we assume that the context $\bc_t$ is drawn by an adversary who may be aware of all previous actions of the learner and past information received by the learner. %This is in line with the assumptions used in the contextual multi-armed bandit setting, see \cite{auer2002using}. 
 We also assume throughout that for any $\bc\in [0,1]^{\dc}$ the minimum $\min_{\bz\in\com} f(\bz, \bc)$ is attained by some $\bz^*(\bc)\in\com$.

We consider strategies for choosing the query points such that $\bz_0\in\mathbb{R}^{\dx}$ is a random variable, and $\bz_t =\pi_t(\{\bc_k,\bz_k,y_k\}_{k=1}^{t-1},\bc_{t},\bzeta_t)$ for $t\geq 1$, where $\pi_t$ is a measurable function with values in $\mathbb{R}^{\dx}$, and $(\bzeta_t)_{t\geq 1}$ is a sequence of random variables with values in a measurable space $(\mathcal{Z},\mathcal{U})$ such that $\bzeta_t$ is independent of $\{\bc_k,\bz_k,y_k\}_{k=1}^{t-1},\bc_{t}$. We call the process $\pi = (\pi_t)_{t\geq 1}$ a randomized procedure and we denote the set of all randomized procedures by~$\Pi$. 

Let $\mathcal{F}'$ be a class of real-valued functions on $\mathbb{R}^d$. We define $\mathcal{F}= \{f: f(\cdot,\bc)\in \mathcal{F}' \text{ for all } \bc\in [0,1]^{p}\}$. Assume that there exists a randomized procedure $\pi$, for which we can control its static regret over $T$ runs for the sequence of functions $\{f(\cdot,\bc_t)\}_{t=1}^{T}$ where  $f\in \mathcal{F}$. Namely, if $\{\bz_t\}_{t=1}^{T}$ are outputs of $\pi$, then there exist $F:\mathbb{R}^{+}\cup\{0\}\to\mathbb{R}^{+}\cup\{0\}$, {$F_1:\mathbb{R}^{+}\cup\{0\}\to\mathbb{R}^{+}\cup\{0\}$ such that $F$ is a concave function, $F_1$ is a non-decreasing function,} and
\begin{align}\label{eq:goodalgox}
{\sup_{f\in \mathcal{F}} \,} \Exp\left[\sum_{t=1}^{T}f(\bz_t,\bc_t) - \min_{\bz\in\com}\sum_{t=1}^{T}f(\bz,\bc_t)\right]\leq F(T){F_1(T)}.
\end{align}
In particular, we will be interested in the case where there exists $(\tau_1,\tau_2,\tau_3,\Tau_0)\in [0,\infty)\times[0,1)\times [0,\infty)\times[1,\infty)$ such that, for $T\geq \Tau_0$,
{
\begin{align}\label{eq:goodalgo1}
    F(T)\le C\dx^{\tau_1}T^{\tau_2} \quad \text{and} \quad F_1(T)\le \log^{\tau_3}(T+1),
\end{align}
}
where $C>0$ is a constant that does not depend on $T$ and $\dx$.
\begin{definition}
We call a randomized procedure that satisfies \eqref{eq:goodalgox} and \eqref{eq:goodalgo1} a $(\mathcal{F},\tau_1,\tau_2,\tau_3,\Tau_0)$-consistent randomized procedure and we denote the set of all such procedures by $\Pi(\mathcal{F},\tau_1,\tau_2,\tau_3,\Tau_0)$. 
\end{definition}
\begin{algorithm}[t!]
    \DontPrintSemicolon
   \caption{}
  \label{algo1}
   \SetKwInput{Input}{Input}
   \SetKwInOut{Output}{Output}
   \SetKwInput{Initialization}{Initialization}
   \Input{Randomized procedure $\pi = (\pi_t)_{t=1}^{\infty}$, parameter $\numbin\in \mathbb{N}$ and a partition $\{B_i\}_{i=1}^{\numbin^d}$ of $[0,1]^{\dc}$ 
   %namely $(B_i)_{i=1}^{\numbin^d}$
   }
   \Initialization{$N_1(0),\dots,N_{\numbin^{\dc}}(0)=0$, $H(1),\dots,H(\numbin^{\dc}) = \{0\}, \text{ and }\bz_0=\bzero$}
   \For{ $t = 1, \ldots, T$}{
   \If{$\bc_t \in B_i$}{
   \vspace{2mm}
   Let $N_i(t) = N_i(t-1) + 1$\tcp*{Increment}
    \vspace{2mm}
   Based on $\bc_t$ and $\{y_k,\bz_k,\bc_k\}_{k\in H(i)}$, use $\pi_{N_i(t)}$ to choose $\bz_t\in\com$ \tcp*{Choosing the query point}
   \vspace{2mm}
   $y_t = f(\bz_t, \bc_t) + \xi_t$\tcp*{Query}

   %   $y_t' = f(\bx_t-h_tr_t\bzeta_t)+ \xi_t',$\tcp*{Query second point}
 
   \vspace{2mm}
   $H(i) = H(i)\cup\{t\}$\tcp*{Update the set of indices}
}
   }
   \end{algorithm}
Algorithm \ref{algo1} is a meta-algorithm that assures regret conversion. It uses a partition $\{B_i\}_{i=1}^{\numbin^{\dc}}$ of the hypercube $[0,1]^{\dc}$, where all the cells $B_i$ are hypercubes of equal volume with edges of length~$1/K$. Fix a consistent randomized procedure $\pi$ for the class of functions $\mathcal{F}$. At each round $t$, if the context~$\bc_t$ falls into the cell $B_i$, Algorithm \ref{algo1} chooses the query point $\bz_t$ by the procedure $\pi$ based only on the past contexts from the same cell $B_i$ and the corresponding past query points. Since no more than $T$ cells $B_i$ can be visited and need to be memorized Algorithm \ref{algo1} is polynomial time provided the input algorithm $\pi$ is polynomial time. 

Let $\{\bc_{i_j}\}_{j=1}^{N_i(T)}$ be all the contexts belonging to $B_i$ revealed up to round $T$, where $N_i(T)$ is the total number of such contexts, and let $\{\bz_{i_j}\}_{j=1}^{N_i(T)}$ be the outputs of $\pi$ restricted to the contexts that belong to $B_i$ (the inner loop of Algorithm \ref{algo1}). This segment of the algorithm involves saving the index of the rounds associated with the received contexts in cell $B_i$. The set $H(i)$ is the collection of such indices. Denote by $\bb_i$ the barycenter of $B_i$, and let $\bx_i^* \in \argmin_{\bx\in\com}f(\bx,\bb_i)$. Then,
\begin{align*}
\Exp\bigg[\sum_{t=1}^{T}\bigg(f(\bz_t,\bc_t) - \min_{\bz\in\com}f(\bz,&\bc_t)\bigg)\bigg] = \Exp\bigg[\sum_{i=1}^{\numbin^{\dc}}\sum_{j=1}^{N_{i}(T)}\left(f(\bz_{i_j},\bc_{i_j}) - f(\bx_{i}^*,\bc_{i_j})+\Delta_{i_j}\right)\bigg] 
\\&\phantom{} \\&\leq \Exp\bigg[\sum_{i=1}^{\numbin^{\dc}}\big(\sum_{j=1}^{N_{i}(T)}f(\bz_{i_j},\bc_{i_j}) - \min_{\bz\in\com}\sum_{j=1}^{N_{i}(T)}f(\bz,\bc_{i_j})+\sum_{j=1}^{N_{i}(T)}\Delta_{i_j}\big)\bigg],
\end{align*}
where $\Delta_{i_j} = f(\bx_{i}^*,\bc_{i_j}) - \min_{\bz\in\com}f(\bz,\bc_{i_j})$. One may notice that
$\Delta_{i_j} \le 2 \max_{\bz\in\com} |f(\bz,\bc_{i_j}) - f(\bz,\bb_i)|$ (see the proof of Theorem \ref{thm:general}), so that
under Assumption \ref{ass:Lip} we have $\Delta_{i_j} \le 2L(\sqrt{p}/ K)^{\gamma}$. Thus, we can control the dynamic regret by the sum of the static regret, which grows with $K$, and the bias term comprised of $\Delta_{i_j}$'s, which decreases with $K$. There exists an optimal $K$ exhibiting a trade-off between the static regret and the bias. In Theorem \ref{thm:general}, we show that if $\numbin$ is tuned  optimally, 
\begin{align}\label{eq:Kvalue}
   \numbin = \numbin(\tau_1,\tau_2,\tau_3,\gamma) \eqdef \max\left(1,\floor{\left(\Lc \dc^{\frac{\gamma}{2}}\dx^{-\tau_1}{T^{(1-\tau_2)}\log^{-\tau_3}(T+1)}\right)^{\frac{1}{\dc(1-\tau_2)+\gamma}}}\right),
    \end{align}
then Algorithm \ref{algo1} achieves a sub-linear dynamic regret. %In what follows, $[T]$ denotes the set of all positive integers less than or equal to $T$.

\begin{theorem}\label{thm:general} Let $(\Lc,\gamma)\in[0,\infty)\times (0,1]$. %Let Assumption \ref{ass:Lip} hold. 
Let $\pi$ be a randomized procedure such that \eqref{eq:goodalgox} holds with a concave function $F$, {a non-decreasing function $F_1$,} and let $\bz_t$'s be the updates of Algorithm~\ref{algo1} with a positive integer $K$.
Then, 
   \begin{align}\label{eq:eboundgenX}
       \sup_{f\in\mathcal{F}\cap \mathcal{F}_\gamma(L)}\Exp\big[\sum_{t=1}^{T}\big(f(\bz_t,\bc_t) - \min_{\bz\in\com}f(\bz,\bc_t)\big)\big]
\leq \numbin^pF\left(\frac{T}{\numbin^p}\right){F_1(T)} + 2LT\left(\frac{\sqrt{\dc}}{\numbin}\right)^{\gamma}.
   \end{align}
  If, in addition, $\pi\in \Pi(\mathcal{F},\tau_1,\tau_2,\tau_3,\Tau_0)$ for a triplet $(\tau_1,\tau_2,\Tau_0)\in [0,\infty)\times[0,1)\times[1,\infty)$,  and $K = \numbin(\tau_1,\tau_2,\tau_3,\gamma)$ (see \eqref{eq:Kvalue}), then for any $T\geq \numbin^p\Tau_0$, we have
   \begin{align}\label{eq:generaldynamo}
&\sup_{f\in\mathcal{F}\cap \mathcal{F}_\gamma(L)}\Exp\big[\sum_{t=1}^{T}\big(f(\bz_t,\bc_t) - \min_{\bz\in\com}f(\bz,\bc_t)\big)\big]\notag\\
&\hspace{1cm}\le C'\max\bigg(\dx^{\tau_1}T^{\tau_2}L_1,\left(\Lc \dc^{\frac{\gamma}{2}}\right)^{\frac{\dc(1-\tau_2)}{\dc(1-\tau_2)+\gamma}}\left(\dx^{\tau_1}L_1\right)^{\frac{\gamma}{\dc(1-\tau_2)+\gamma}}T^{\frac{(\dc-\gamma)(1-\tau_2) + \gamma}{\dc(1-\tau_2)+\gamma}}\bigg).
\end{align}
where $L_1 = \log^{\tau_3}(T+1)$ and $C'>0$ is a constant %that appeared in \eqref{eq:goodalgo1} and is 
independent of $T$, $\dx$, and $\dc$.
\end{theorem}

Note that the Lipschitz constant $\Lc$ measures how sensitive the function $f$ is with respect to variations of the contexts. If $\Lc = 0$ the static and contextual regrets are equivalent, and it is not surprising that \eqref{eq:goodalgo1} and \eqref{eq:generaldynamo} have the same rates. We also note that a bound that depends in the same way as \eqref{eq:generaldynamo} on the parameters $\dc$, $\dx$, and $T$ can be achieved with a choice of $K$ independent of $\Lc$.   In deriving \eqref{eq:generaldynamo} we assumed that the learner knows the value of $\Lc$ but this is only done to illustrate the above remark. % about the connection to the case $L=0$.% about the notions of static and contextual regret. 

The approach to contextual bandits based on discretizing the space of contexts and running some algorithms $\pi$ separately {for each discretization item} is quite common, see \cite{hazan2007online,
kleinberg2008multi, bubeck2008online, slivkins14a, cesa2017algorithmic,li2019dimension}. With the exception of \cite{slivkins14a}, each of these papers, as well as follow-up works based on similar arguments, explored a specific algorithm $\pi$, a specific class $\mathcal{F}$, and provided bounds on the contextual regret (or on $L$-regret) using their particular properties. More generally, \cite[Theorem 21]{slivkins14a} provides a static-to-contextual regret conversion device, where the class $\mathcal{F}$ is still quite specific (defined by conditions (S1) and (S2) below) but the input algorithm $\pi$ is not fixed and allowed to belong to the class of algorithms with $\tilde O(T^{\frac{1}{p+2}})$ static regret.  
%
%To our knowledge, the only result invoking a general conversion from static to contextual regret is {\color{purple}{Theorem 21}}\karim{I did not find any theorem 8.2 in this reference. Did you mean Section 8.2 in slivkins 2014 \cite{slivkins14a}?} in \cite{slivkins14a}. However, it applies to a specific class $\mathcal{F}$ \AT{that is not easy to conceive} (see (S1) and (S2) below) 
The bound on the contextual regret in \cite[Theorem 21]{slivkins14a} includes the doubling dimension, which is exponential in the dimension of the context. In contrast, Theorem~\ref{thm:general} is a conversion result valid for a general class $\mathcal{F}$, a general input algorithm $\pi$,  and its  bound is free from exponential factors. It yields a ready-to-use tool for exploring various particular settings.   
Below we demonstrate its application to the three major contextual settings, where the class $\mathcal{F}$ is defined by Assumption~\ref{ass:Lip} combined with one of the following assumptions on the class of functions $f(\cdot, \bc)$ for any fixed $\bc$:
\begin{itemize}
    \item[(a)] $\mathcal{F}'$  is a class of Lipschitz functions,
    \item[(b)]  $\mathcal{F}'$  is the class of all convex functions,
    \item[(c)]  $\mathcal{F}'$ is a class of strongly convex and smooth functions.
\end{itemize}
{For the sake of comparison, we recall that the class of loss functions $\mathcal{F}$ considered in \cite[Theorem 21]{slivkins14a} is defined by the following two conditions:\\
\phantom{00} (S1) $f(\bz,\bc)$ is Lipschitz in $(\bz,\bc)$ with Lipschitz constant 1,\\
\phantom{00} (S2) for each context $\bc$, there exists $\bz^*(\bc)\in \argmin_{\bz\in\com} f(\bz,\bc)$ such that the map $\bc\mapsto f(\bz^*(\bc),\bc)$ is Lipschitz, with Lipschitz constant 1. \\
\phantom{0} Condition (S2) does not follow from (S1) and requires some strong additional properties of $f$ that are not detailed in \cite{slivkins14a}. Moreover, as emphasized in \cite{slivkins14a} the assumption that the Lipschitz constant is 1 can be relaxed but in that case the algorithm should know the Lipschitz constant.} 

%for one of such examples, where $\mathcal{F}'$ is the class of strongly convex and smooth functions.} %there exist $(\tau_1,\tau_2,\Tau_0)\in [0,
%\infty)\times [0,1) \times [1,\infty)$ such that $\Pi(\mathcal{F},\tau_1,\tau_2,\Tau_0)$ is non-empty.

For the rest of this section, we focus on the implications of Theorem \ref{thm:general} for two families of functions corresponding to items (a) and (b) above. Observe that considering (a) contains as a special case the class $\mathcal{F}$ of functions $f(\bz,\bc)$ that are Lipschitz in $(\bz,\bc)$ with no further restriction. We do not need the hardly interpretable additional condition (S2), nor the knowledge of the Lipschitz constant.  

{%\color{purple}
%\textbf{Lipschitz Bandits.}
\hspace{2mm}(a) {\bf Lipschitz bandits.} Lipschitz non-contextual bandits have been studied extensively in the literature, cf. \cite{lattimore2020bandit} and the references therein. The main idea is to reduce the Lipschitz bandit problem to a multi-armed bandit problem with finite number of arms, often employing zooming techniques with dyadic splits; see, for example, \cite{kleinberg2004nearly,kleinberg2008multi,bubeck2008online}. While these methods benefit from near optimal  in 
$T$ static regret (of the order  $\tilde O(T^{\frac{\dx+1}{\dx+2}})$), the regret bounds depend exponentially on the dimension $d$.
A recent improvement due to \cite{podimata2021adaptive} shows that polynomial dependence on the dimension is achievable. In \cite{podimata2021adaptive}, the authors proposed an algorithm with static regret that scales as $\dx^{\frac{1}{2}}T^{\frac{\dx+1}{\dx+2}}\log^{5}(T)$ for $\com = [0,1]^{\dx}$ (cf. \cite[Theorem 2]{podimata2021adaptive} with $z=\dx$ for the worst-case scenario). Taking their method as an input $\pi$ of our meta-algorithm and combining the bound of \cite[Theorem 2]{podimata2021adaptive} with Theorem \ref{thm:general} we obtain the following corollary. 

\begin{corollary}\label{cor:lip} Assume that $\com = [0,1]^{\dx}$ and that the observations in \eqref{eq:observe} are noise-free, that is, $\xi_t = 0$ for all positive integers $t$. Let $\mathcal{F}$ be the set of functions such that 
$|f(\bx,\bc) - f(\bx',\bc)|\leq \Lc'\norm{\bx-\bx'}$ for all $\bx,\bx'\in[0,1]^d$,  $\bc\in[0,1]^{\dc}$, where $L'>0$ is a constant. Let $\bz_t$’s be the updates of Algorithm \ref{algo1} with the input algorithm $\pi$ taken as the Lipschitz bandit algorithm from \cite{podimata2021adaptive}. Then, under the optimal choice of $K$ we have 
\begin{align*}
    \sup_{f\in\mathcal{F}\cap \mathcal{F}_\gamma(L)}\Exp\big[\sum_{t=1}^{T}\big(f(\bz_t,\bc_t) - \min_{\bz\in\com}f(\bz,\bc_t)&\big)\big]\leq C\max\bigg(\dx^{1/2}T^{(\dx+1)/(\dx+2)}L_1,\\&\big(\dc^{\gamma\dc/2}\big(\dx^{\frac{1}{2}}L_1\big)^{\gamma(\dx+2)}T^{\dc+(\dx+1)\gamma}\big)^{1/(\dc + (\dx+2)\gamma)}\bigg),
\end{align*}
where $L_1 = \log^5(T+1)$, and $C>0$ is a constant that does not depend on $T$, $\dx$, and $\dc$.
\end{corollary}
\begin{proof}
By \cite[Theorem 2]{podimata2021adaptive} the input algorithm $\pi$ satisfies the static regret bound \eqref{eq:goodalgox} with
\begin{align*}
F(T) \leq C_1\dx^{1/2}T^{(\dx+1)/(\dx+2)},\quad \text{and}\quad F_1(T+1) \leq \log^5(T+1),
\end{align*}
where $C_1 > 0$ is a constant independent of $T$, $\dx$, and $\dc$. Thus, we can apply Theorem \ref{thm:general} with $\tau_1 = 1/2$, $\tau_2 = (\dx+1)/(\dx+2)$, and $\tau_3 = 5$, which yields the result.%for sufficiently large $T$, the contextual regret is guaranteed to be of the order of
%
%\begin{align}\label{eq:context_Lip}
%\max\bigg(\dx^{1/2}T^{(\dx+1)/(\dx+2)}L_1,\big(\dc^{\gamma\dc/2}\big(\dx^{\frac{1}{2}}L_1\big)^{\gamma(\dx+2)}T^{\dc+(\dx+1)\gamma}\big)^{1/(\dc + (\dx+2)\gamma)}\bigg).
%\end{align}
%
\end{proof}
In particular, in the case $\gamma=1$, that is, when $f$ is Lipschitz continuous in both arguments, Corollary \ref{cor:lip} establishes the rate in $T$ as $T^{(\dc+\dx+1)/(\dc+\dx+2)}$ up to a logarithmic factor and a weak factor depending on $p,d$, which does not exceed $O(\sqrt{pd})$. On the other hand, for $\gamma=1$, \cite{slivkins14a} provides a matching lower bound on the contextual regret of the same order $T^{(\dc+\dx+1)/(\dc+\dx+2)}$. Together with Corollary~\ref{cor:lip}, it implies that the rate $T^{(\dc+\dx+1)/(\dc+\dx+2)}$ is minimax optimal as function of $T$, up to logarithmic factors, and our algorithm attains the minimax rate. Notice that the minimax optimality of this rate on the class of functions $f$ that are Lipschitz continuous in both arguments was not established in \cite{slivkins14a}, contrary to what is sometimes claimed. Indeed, \cite{slivkins14a} provided an upper bound on a substantially smaller class where additionally condition (S2) holds. Furthermore, the upper bound in \cite[Theorem 21]{slivkins14a} is proportional to doubling dimension, which scales exponentially in $p$, while Corollary~\ref{cor:lip} involves only a weak dimension dependent factor that never exceeds $O(\sqrt{pd})$.

%Next, we focus on the implications of Theorem \ref{thm:general} for convex contextual bandits.

}
\hspace{2mm} (b) {\bf Convex bandits.} %Convex bandits represent a family of constrained convex online optimization problems where, at each round, the learner can only observe a single function evaluation, and the query point must belong to the constraint set $\Theta$. 
The study of non-contextual bandit convex optimization was initiated by \cite{flaxman2004online} and \cite{kleinberg2004nearly}. There exists an extensive literature on this problem for both stochastic and adversarial settings, see \cite{orabona2019modern,lattimore2020bandit,lattimore2024bandit} and the references therein. Bounds on the static regret in adversarial setting were developed in \cite{bubeck2017kernel, Bubeck2018ExploratoryDF, lattimore2020bandit,fokkema2024online}. Papers \cite{bubeck2017kernel, Bubeck2018ExploratoryDF, lattimore2020bandit} considered the setting with no noise ($\xi_t=0$ for all $t$) providing successive improvements of static regret bounds. The best known result \cite{lattimore2020bandit} yields a bound of the order $d^{2.5}\sqrt{T}$  derived as an evaluation of the minimax regret without proposing an explicit algorithm (here and below we ignore the logarithmic factors). An important progress was recently achieved in \cite{fokkema2024online}, where the authors proposed a polynomial time algorithm attaining a regret of the order $\dx^{3.5}\sqrt{T}$ in adversarial setting with sub-Gaussian noise, cf. \cite[Theorem~1]{fokkema2024online}. Combining this result, which is the state of the art, with Theorem~\ref{thm:general} we obtain Corollary~\ref{cor:conv} below for the contextual convex bandit setting.  As \cite[Theorem~1]{fokkema2024online} does not explicitly detail the dependence on logarithmic factors, we adopt the same convention and present the bound up to logarithmic factors.

\begin{corollary}\label{cor:conv}Let $\mathcal{F}$ be the set of functions $f$ such that $f(\cdot,\bc)$ is convex for all $\bc\in[0,1]^{\dc}$. Assume that, conditionally on the past observations, $\xi_t$ in \eqref{eq:observe} is a zero-mean $\sigma$-sub-Gaussian random variable, that is, for any $\lambda\in {\mathbb R}$ and some $\sigma>0$ we have $\Exp\left[\exp\left(\lambda \xi_t\right)| \{\bc_k,\bz_k,y_k\}_{k=1}^{t-1},\bc_{t}\right]\leq \exp\left(\sigma^2 \lambda^2/2\right)$\footnote{Requiring the inequality for all $\lambda\in {\mathbb R}$ grants that $\xi_t$ has zero conditional mean.} for all positive integers $t$. Let $\bz_t$’s be the updates of Algorithm \ref{algo1} with the input algorithm $\pi$ taken as the convex bandit algorithm from \cite{fokkema2024online}. Then, under the optimal choice of $K$ we have
\begin{align*}
     \sup_{f\in\mathcal{F}\cap \mathcal{F}_\gamma(L)}\Exp\big[\sum_{t=1}^{T}\big(f(\bz_t,\bc_t) - \min_{\bz\in\com}f(\bz,\bc_t)&\big)\big]\leq C\max\left(\dx^{7/2}\sqrt{T}, \dc^{\gamma\dc/(2\dc+4\gamma)}\dx^{7\gamma/\left(\dc + 2\gamma\right)}T^{\frac{\dc+\gamma}{\dc+2\gamma}}\right),
\end{align*}
where $C>0$ is a factor that depends polynomially on $\log(T)$ and $\log(d)$ and does not depend on $\dc$.
\end{corollary} 
\begin{proof}
    By \cite[Theorem 1]{fokkema2024online} the input algorithm $\pi$ satisfies the static regret bound \eqref{eq:goodalgox} with  $F(T) \leq C'\dx^{7/2}\sqrt{T}$ and $F_1(T)= \log^a(T)$ for some $a>0$ and a constant $C'>0$ that depends polynomially on $\log(d)$ and does not depend on $\dc$. Thus, we can apply Theorem~\ref{thm:general} with $\tau_1 = 7/2$ and $\tau_2 = 1/2$. Taking $\numbin$ as defined in \eqref{eq:Kvalue} with $\tau_1 = 7/2$, $\tau_2 = 1/2$, and $\tau_3 = a$ we obtain the result of the theorem (with the dependency with respect to logarithmic factors hidden in $C$).  
\end{proof}
Thus, for convex contextual bandits we have an algorithm with the contextual regret that scales in $T$ as $T^{\frac{\dc+\gamma}{\dc+2\gamma}}$ up to a logarithmic factor. The dependence on $d$ and $p$ is mild. It is given by a factor that never exceeds $p^{1/2} d^{7/3}$, to within a power of $\log(d)$. Combining this result with the lower bound of Theorem \ref{thm:lw_STRc_cont} (see Section \ref{sec:lower} below) we conclude that the minimax optimal rate in $T$ for convex contextual bandits is of the order $T^{\frac{\dc+\gamma}{\dc+2\gamma}}$ up to a logarithmic factor.

\section{Strongly convex and smooth functions}\label{sec:strong}

In this section, we apply the regret conversion approach developed in Section \ref{sec:general} to the contextual bandits problem with the objective functions $f$ such that $f(\cdot, \bc)$ is $\beta$-smooth and $\alpha$-strongly convex for any $\bc \in [0,1]^p$. 
\begin{definition}\label{def:smooth}
 Let $\Lx>0$. Function $g:\mathbb{R}^{\dx}\to\mathbb{R}$ is called $\Lx$-smooth, if it is continuously differentiable and for any $\bx,\by\in\mathbb{R}^{\dx}$ it satisfies
 \begin{align*}
     \norm{\nabla g(\bx) - \nabla g(\by)} \leq \Lx\norm{\bx-\by}.
 \end{align*}
\end{definition}

\begin{definition}
    Let $\alpha>0$. A differentiable function $g:\mathbb{R}^{\dx}\to\mathbb{R}$ is called $\alpha$-strongly convex, if for any $\bx,\by\in\mathbb{R}^{\dx}$ it satisfies
    \begin{align*}
        g(\bx)\geq g(\by) +\langle\nabla g(\by), \bx-\by\rangle+\frac{\alpha}{2}\norm{\bx-\by}^{2}.
    \end{align*}
\end{definition}

We impose the following assumption on the function $f$. 

\begin{assumption}\label{ass:Hstrconvex}
    For $(\alpha,\Lx,M)\in (0,\infty)\times[\alpha,\infty)\times(0,\infty)$, we assume that $f\in \mathcal{F}_{\alpha,\Lx}(M)$, where $\mathcal{F}_{\alpha,\Lx}(M)$ denotes the class of all functions $f:\mathbb{R}^{\dx}\times [0,1]^{\dc}\to\mathbb{R}$ such that for all  $\bc\in[0,1]^{\dc}$ the following holds: 
    \begin{itemize}
    \item[i)] $f(\cdot,\bc)$ is $\alpha$-strongly convex. 
        \item[ii)] $f(\cdot,\bc)$ is $\Lx$-smooth.
        \item[iii)] $\max_{\bx\in\com}|f(\bx,\bc)|\leq M$.
    \end{itemize}
\end{assumption}

The non-contextual bandit setting with strongly convex and smooth functions has been investigated in \cite{hazan2014bandit} and \cite{ito2020optimal}. Both \cite{hazan2014bandit} and \cite{ito2020optimal} studied the problem with no noise ($\xi_t\equiv 0$ for all $t$). In \cite{hazan2014bandit}, the authors proposed an algorithm leveraging an interior point method and a self-concordant function. They achieved a static regret, which is nearly optimal, ranging from $d\sqrt{T}$ to $d^{3/2}\sqrt{T}$ depending on the geometry of the constraint set. Assuming that all minimizers of the objective functions are in the interior 
of the constraint set, \cite{ito2020optimal} introduced a procedure that achieves static regret of the optimal order 
$d\sqrt{T}$ (however, with no polynomial time guarantee).   In this section, we extend the algorithm of \cite{hazan2014bandit} to the setting with noisy function evaluation as in \eqref{eq:observe} and derive a bound on the static regret of the proposed algorithm. Then, using it as an input algorithm $\pi$ of our meta-algorithm we obtain a bound on the contextual regret of the resulting procedure.

%Leveraging Theorem \ref{thm:general}, fulfillment of Assumption \ref{ass:Lip} guarantees that if for some $(\tau_1,\tau_2,\Tau_0)\in [0,\infty)\times [0,1)\times [1,\infty)$, $\Pi(\mathcal{F}, \tau_1, \tau_2, \Tau_0)$ is non-empty, then attaining a sub-linear dynamic regret is feasible. First, we impose the assumptions that aid us in constructing an algorithm with such a quality. The first definition regards the gradient of the function, which we assume is a Lipschitz function. 

Following the approach of Section \ref{sec:general}, we first construct a consistent randomized procedure, that is, a procedure with sub-linear static regret, cf. \eqref{eq:goodalgox}. To this end, our starting point is the BCO algorithm introduced in \cite{hazan2014bandit}. The BCO algorithm makes use of the interior point method, which is a vital tool in convex optimization, see, e.g, \cite{nesterov1994interior}. A pioneering role of introducing this method to the field of online learning is credited to \cite{abernethy2008competing}. 
We cannot directly apply the BCO algorithm as defined in \cite{hazan2014bandit} since \cite{hazan2014bandit} deals with the setting, where the observations $y_t$  are noise-free. Therefore, we propose a modification of the BCO algorithm adapted to noisy observations \eqref{eq:observe} with sub-Gaussian noise $\xi_t$. It is defined in Algorithm \ref{algoBCO}. %In addition, we assume that the observation noise $\xi_t$ in \eqref{eq:observe} is independent of the randomization $\bzeta_t$ induced by the algorithm. 
In what follows, $[T]$ denotes the set of all positive integers less than or equal to $T$. We denote the unit Euclidean ball in $\mathbb{R}^{\dx}$ by $B^{\dx}$ and the unit Euclidean sphere by $\partial B^{\dx}$.

\begin{assumption}\label{ass:noise}
    %{\color{magenta} The random variables $\xi_1,\dots,\xi_T$ are mutually independent, and} 
    There is $\sigma>0$ such that for all $t\in[T]$ the following holds:  (i) $\xi_t$ is $\sigma$-sub-Gaussian, i.e., for any $\lambda>0$ we have $\Exp\left[\exp\left( \lambda\xi\right)\right]\leq \exp\left(\sigma^2 \lambda^2/2\right)$; (ii) $\xi_t$ is independent of $\bzeta_t$.
\end{assumption}
\vspace{-2mm}
We emphasize that the mutual independence of the noise variables $\xi_t$ is not assumed.

The main ingredient of the interior point method is a self-concordant barrier defined as follows.

\begin{definition}
    Let $\text{int}(\com)$ be the interior of the set $\com$. For $\mu>0$, a function $\mathcal{R}:\text{int}(\com)\to\mathbb{R}$ is a self-concordant function if the following holds.
    \begin{itemize}
        \item $\mathcal{R}$ is three times differentiable and converges to infinity for any trajectory that converges to the boundary of $\com$, i.e., $\mathcal{R}(\by) \to \infty$ as $\by\to \partial \com$.
        \item For any $\bx\in\mathbb{R}^{\dx}$ and $\by\in\text{int}(\com)$ we have
        \begin{align*}
            |\nabla^3 \mathcal{R}(\by)[\bx,\bx,\bx]|\leq 2\left(\nabla^2 \mathcal{R}(\by)[\bx,\bx]\right)^{\frac{3}{2}},
        \end{align*}
        where for $m\in{1,2,3}$, and $\bv_1,\dots,\bv_m\in\mathbb{R}^{\dx}$\begin{align*}
            \nabla^m\mathcal{R}(\by)[\bv_1,\dots,\bv_m] = \frac{\partial^m}{\partial t_1\dots \partial t_m}\mathcal{R}\left(\by + t_1\bv_1+\dots+t_m\bv_m\right)|_{t_1=0,\dots,t_m=0}.\end{align*}
            \end{itemize}
            If, in addition, for some $\mu>0$, and any $\bx\in\mathbb{R}^{\dx}$ and $\by\in\text{int}(\com)$ we have 
\begin{align*}
         |\nabla \mathcal{R}(\by)[\bx]|\leq \mu^{\frac{1}{2}}\left(\nabla^2 \mathcal{R}(\by)[\bx,\bx]\right)^{\frac{1}{2}},
        \end{align*}
     then the function $\mathcal{R}$ is called a $\mu$-self-concordant barrier.
\end{definition}

Since $\com$ is bounded, $\nabla^2 \mathcal{R}(\bx)$ is positive definite, the minimizer of $\mathcal{R}$ in $\text{int}(\com)$ is unique, see \cite[Chapter 3, Section V]{nemirovski2004interior}, and $\bx_{t}$ in Algorithm \ref{algoBCO} is uniquely defined as well. 
Moreover, without loss of generality, we assume that $\min_{\bx\in\text{int}(\com)}\mathcal{R}(\bx) = 0$. Self-concordant barriers can be chosen depending on the geometry of $\com$. For the examples of $\mathcal{R}$ and the corresponding values of $\mu$, we refer the reader to Appendix~\ref{app:self-concord}.

\begin{algorithm}[t!]
    \DontPrintSemicolon
   \caption{}
  \label{algoBCO}
   \SetKwInput{Input}{Input}
   \SetKwInOut{Output}{Output}
   \SetKwInput{Initialization}{Initialization}
   \Input{$\mu$-self-concordant barrier function $\mathcal{R}$ on $\com$, and $q_T = M+2\sigma\sqrt{\log(T+1)}$}
   \Initialization{$\bx_0 = \argmin_{\bx\in\text{int}(\com)} \mathcal{R}(\bx)$}
   \For{ $t = 1, \ldots, T$}{
  
      $\eta_t = (4\dx q_T)^{-1}\min\left(1,\sqrt{\frac{16(\mu + \Lx/\alpha)\log(t+1)}{t}}\right)$\tcp*{Update the step size}
       $P_t = \left(\nabla^2\mathcal{R}(\bx_{t-1})+\eta_t \alpha t\mathbf{I}_{\dx}\right)^{-\frac{1}{2}} $\tcp*{Update the perturbation parameter}
       Generate $\bzeta_t$ uniformly distributed on $\partial B^{\dx}$\tcp*{Randomization}
       $\bz_t= \bx_{t-1} + P_t\bzeta_t$\tcp*{Output}
   $y_t = f(\bz_t,\bc_t)+ \xi_t\quad$ \tcp*{Query}

   %   $y_t' = f(\bx_t-h_tr_t\bzeta_t)+ \xi_t',$\tcp*{Query second point}
   $\bg_t = \dx y_t P_t^{-1}\bzeta_t $\tcp*{Estimate the gradient}
   
   $\bx_{t} = \argmin_{\bx\in\com}\eta_t\sum_{k=1}^{t}\left(\langle \bg_{k},\bx\rangle +\frac{\alpha}{2}\norm{\bx-\bx_{k-1}}^2\right) +\mathcal{R}(\bx)$\tcp*{Iterate}
   
   }
\end{algorithm}

%For $\mu>0$, let $\mathcal{R}$ be a $\mu$-self-concordant barrier on $\com$. We start with the initialization $\bx_0\in\argmin_{\bx\in\text{int}(\com)}\mathcal{R}(\bx)$. Let $q_T = M+\sigma\sqrt{2\log(T+1)})$, and note

The difference between Algorithm \ref{algoBCO} and the BCO algorithm of \cite{hazan2014bandit} is in the assigned step size. In \cite{hazan2014bandit} the step size is constant and depends on the horizon $T$ while it varies with $t$ in Algorithm \ref{algoBCO}. Next, in Algorithm \ref{algoBCO}, the step size is downscaled by $q_T = M+2\sigma\sqrt{\log(T+1)}$ whereas in \cite{hazan2014bandit} it is downscaled by $M$ without any additional logarithmic term. This difference is due to the fact that the observations in \eqref{eq:observe} are noisy whereas in \cite{hazan2014bandit} they are assumed to be noise-free. In our analysis, we rely on the fact that $\eta_t\norm{P_t\bg_t}$ is uniformly bounded with high probability. To ensure this, it is enough to scale down the step size by adding the above logarithmic term to $M$. Indeed, due to Assumption \ref{ass:noise}, we have  that $\max_{t\in[T]}|y_t|\leq q_T$ with a suitably high probability (see the details in the proof of Theorem \ref{thm:statstrconv}).  %(see \eqref{eq:varbound} in Lemma \ref{lem:mainopt1}).

Algorithm \ref{algoBCO} algorithm selects 
$\bz_t\in\{\bx_{t-1} + P_t\bzeta|\bzeta \in \partial B^{\dx}\} \subseteq D\eqdef\{\bz|(\bz - \bx_{t-1})^{\top}\nabla^2\mathcal{R}(\bx_{t-1})(\bz - \bx_{t-1})\leq 1\}$, in which $D$ is the Dikin ellipsoid. Consequently, $D\subseteq\com$; see \cite[Proposition 14.2.2]{nemirovski2004interior} for a proof. Therefore, $\bz_t \in \com$ for all $t\in [T]$.
{
\begin{theorem}\label{thm:statstrconv}
Let Assumptions \ref{ass:Hstrconvex} and \ref{ass:noise}  hold. Then, the static regret of Algorithm \ref{algoBCO} satisfies
\begin{align*}  \Exp\left[\sum_{t=1}^{T}f(\bz_t,\bc_t) - \min_{\bz\in\com}\sum_{t=1}^{T}f(\bz,\bc_t)\right]\leq F(T)F_1(T),
\end{align*}
where $F(\cdot)$ is a concave function defined for $x\ge 0$ by 
\begin{align*}
    F(x) = 11\sqrt{\nu x \log(x+1)}+\nu\left(\log(x+1) + 2\log(\nu+1)\right),
\end{align*}
and $F_1(\cdot)$ is a non-decreasing function such that for $x\ge 0$ 
\begin{align*}
        F_{1}(x)=\frac{\dx}{4}\left(5.3M+2\sigma\sqrt{\log(x+1)}\right),
\end{align*}
with $\nu = 16(\mu + \Lx/\alpha)$. 
Moreover, for $T \geq \nu^2/16$, we have 
\begin{align}\label{eq:str_stat_reg}   F(T)\leq 17\sqrt{\nu T\log(T+1)},
\end{align}
and subsequently
\begin{align}\label{eq:concaveboundstr}   F(T)F_1(T)\leq 17\dx\left(5.3M+2\sigma\sqrt{\log(T+1)}\right)\left(\left(\mu+\frac{\beta}{\alpha}\right) T\log(T+1)\right).
\end{align}
\end{theorem}
}

For the case of noise-free observations in \eqref{eq:observe} we have 
$\sigma = 0$, and \eqref{eq:str_stat_reg} has the same order as the bound on the regret provided in \cite{hazan2014bandit}.  %Therefore, in the case of an unknown horizon, \eqref{eq:str_stat_reg} is achieved without using the doubling trick. 
We also note that, for the case of bounded noise ($|\xi_t|\le \sigma$), $q_T$ can be set as 
$q_T = M+\sigma$ and a refined regret in \eqref{eq:str_stat_reg} can be derived, where the term 
$2\sigma\sqrt{\log(T+1)}$ is replaced by 
$\sigma$. Finally, even for the case of no noise  
$\sigma = 0$ we improve on the result of \cite{hazan2014bandit} since, unlike \cite{hazan2014bandit}, our Algorithm \ref{algoBCO} is anytime and does not require the knowledge of 
$T$ when $\sigma = 0$.

%Finally, for the general case of sub-Gaussian noise, \eqref{eq:str_stat_reg} includes an additional multiplicative factor of $\sqrt{\log(T+1)}$ compared to the bounded noise  and noise-free observation framework. %This additional factor arises from controlling the noise in the observation.

%\begin{rem}
    
%\end{rem}

It is a known fact (see Appendix \ref{app:self-concord}) that, for any convex body $\com$, one can construct a $\mu$-self-concordant barrier such that for $\dxmu\in [0,1]$, $\mu \leq \cst d^{\dxmu}$, where $\cst > 0$ is a constant depending only on $\com$ and independent of $\dx$. This fact and Theorem~\ref{thm:statstrconv} imply that $\Pi\left(\mathcal{F}_{\alpha,\Lx}(M),1 + \dxmu/2,1/2,1, 16(\mu+\Lx/\alpha)^2\right)$ is non-empty. In \cite{shamir2013complexity,akhavan2020exploiting,akhavan2023gradient}, the authors established a lower bound of the order of $\dx/\sqrt{T}$ for the optimization error when the noise is Gaussian (\cite{shamir2013complexity}) or satisfies more general assumptions (\cite{akhavan2020exploiting,akhavan2023gradient}). By convexity, this lower bound, when multiplied by $T$, provides a lower bound for the static regret in the case $\Lc = 0$, that is, when $f(\bx,\bc_1) = f(\bx,\bc_t)$ for all $\bx \in \mathbb{R}^{\dx}$ and $t \in [T]$. Thus, the lower bound in \cite{shamir2013complexity,akhavan2020exploiting,akhavan2023gradient} results in a valid lower bound of the order of $d\sqrt{T}$ for the static regret in the noisy setting. Therefore, if $\dxmu = 0$, which is the case when $\com$ is a Euclidean ball, \eqref{eq:str_stat_reg} is optimal up to a logarithmic factor. In the worst case scenario ($\dxmu= 1$), for instance, when $\com$ is the $\ell_{\infty}$-ball or the simplex (see Appendix \ref{app:self-concord}), akin to \cite{hazan2014bandit}, our algorithm exhibits an optimality gap of at most $\dx^{1/2}$ up to logarithmic factors.

In order to enable the regret conversion via Theorem \ref{thm:general}, we need Assumption \ref{ass:Lip}, that is, the premise that $f$ belongs to $\mathcal{F}_\gamma(L)$. Since it always holds that $\mu \leq \cst d^{\dxmu}$, where $\dxmu\in[0,1]$ and $\cst > 0$ is a numerical constant, %Theorem~\ref{thm:statstrconv} implies that Algorithm \ref{algoBCO} belongs to the class $\Pi\left(\mathcal{F}_{\alpha,\Lx}(M),1 + \dxmu/2,1/2,16(\mu+\Lx/\alpha)^2\right)$. Using this remark 
we obtain the following corollary of Theorems \ref{thm:general} and \ref{thm:statstrconv}. 
\begin{algorithm}[t!]
    \DontPrintSemicolon
   \caption{}
  \label{algo2}
   \SetKwInput{Input}{Input}
   \SetKwInOut{Output}{Output}
   \SetKwInput{Initialization}{Initialization}
   \Input{Parameter $\numbin\in \mathbb{N}$ and a partition of $[0,1]^{\dc}$ namely $(B_i)_{i=1}^{\numbin^{\dc}}$, $\mu$-self-concordant barrier $\mathcal{R}$}
   \Initialization{$N_1(0),\dots,N_{\numbin^{\dc}}(0)=0$, $H(1),\dots,H(\numbin^{\dc}) = \{0\}$, $\bx_0 = \argmin_{\bx\in\text{int}(\com)} \mathcal{R}(\bx)$}
   \For{ $t = 1, \ldots, T$}{
   \If{$\bc_t \in B_i$}{
      $N_i(t) = N_i(t-1)+1$\tcp*{Increment}
      Let $q = \max H(i)$\tcp*{The index of the last event in $B_i$}
      $\eta_t = (4\dx q_T)^{-1}\min\left(1,\sqrt{\frac{16(\mu + \Lx/\alpha)\log(N_i(t)+1)}{N_i(t)}}\right)$\tcp*{Update the step size}
       $P_t = \left(\nabla^2\mathcal{R}(\bx_q)+\eta_t \alpha N_i(t)\mathbf{I}_{\dx}\right)^{-\frac{1}{2}} $\tcp*{Update the perturbation parameter}
       Generate $\bzeta_t$ from $\partial B^{\dx}$\tcp*{Randomization}
       $\bz_t= \bx_q + P_t\bzeta_t$\tcp*{Output}
   $y_t = f(\bz_t,\bc_t)+ \xi_t\quad$ \tcp*{Query}

   %   $y_t' = f(\bx_t-h_tr_t\bzeta_t)+ \xi_t',$\tcp*{Query second point}
   $\bg_t = \dx y_t P_t^{-1}\bzeta_t $\tcp*{Estimate the gradient}
   
   $\bx_{t} = \argmin_{\bx\in\com}\eta_t\sum_{k\in H(i)}\left(\langle \bg_{k},\bx\rangle +\frac{\alpha}{2}\norm{\bx-\bx_{k-1}}^2\right) +\mathcal{R}(\bx)$\tcp*{Iterate}
   
   $H(i) = H(i)\cup {t}$\tcp*{Save the index of last event in $B_i$}
   }
   }
\end{algorithm}
\begin{corollary}\label{cor:dynamocstrconvex}
Let Assumptions \ref{ass:Lip}, \ref{ass:Hstrconvex}, and \vspace{1mm}\ref{ass:noise} hold. 
    Let $\dxmu \in [0,1]$ be such that $\mu \leq \cst d^{\dxmu}$, where $\cst > 0$ is a numerical constant.
    Then, the dynamic regret of Algorithm \ref{algo2} with
        $\numbin= \numbin(1+\dxmu/2,1/2,1, \gamma)$ (see \eqref{eq:Kvalue}),
satisfies
\begin{align*}
 &   \Exp\big[\sum_{t=1}^{T}\big(f(\bz_t,\bc_t) - \min_{\bz\in\com}f(\bz,\bc_t)\big)\big]
\le C \max\bigg\{L_1\dx^{1+\dxmu} + L_2\sqrt{\dx^{2+\dxmu} T}, \, 
    \\&
    \qquad \qquad \qquad \qquad \qquad 
   L_1\Big(\Lc\dc^{\frac{\gamma}{2}}L_2^{-1}\Big)^{\frac{2\dc}{\dc+2\gamma}}
    \dx^{\frac{2\gamma(1+\dxmu)-\dc}{\dc+2\gamma}}
    T^{\frac{\dc}{\dc+2\gamma}},
    \,
    L_2^\frac{2\gamma}{\dc+2\gamma}\Big(\Lc \dc^{\frac{\gamma}{2}}\Big)^{\frac{\dc}{\dc+2\gamma}}
    \dx^{\frac{(2+\dxmu)\gamma}{\dc+2\gamma}}
    T^{\frac{\dc + \gamma}{\dc+2\gamma}}  
    \bigg\},
\end{align*}
where $L_1=\sqrt{\log (T+1)}\log (Td+1)$, $L_2=\log(T+1)$, and $C>0$ is a constant depending only on $M,\sigma, \Lx/\alpha$.

%\begin{align*}
%    \Exp\bigg[\sum_{t=1}^{T}\bigg(f(\bz_t,\bc_t)-\min_{\bz\in\com}f(\bz,&\bc_t)\bigg)\bigg]
 %   \lesssim 
%    \max\bigg\{\sqrt{\dx^{2+\dxmu}\,T} +\dx^{1+\dxmu}, \, 
%    \Big(\Lc\dc^{\frac{\gamma}{2}}\Big)^{\frac{2\dc}{\dc+2\gamma}}\dx^{\frac{2\gamma(1+\dxmu)-\dc}{\dc+2\gamma}}T^{\frac{\dc}{\dc+2\gamma}},
%    \\&
%    \qquad \qquad \qquad
%    \Big(\Lc \dc^{\frac{\gamma}{2}}\Big)^{\frac{\dc}{\dc+2\gamma}}\dx^{\frac{(2+\dxmu)\gamma}{\dc+2\gamma}}T^{\frac{\dc + \gamma}{\dc+2\gamma}}  
%    \bigg\}.
%\end{align*}
\end{corollary}
%\arya{In \cite{li2019dimension}  the claim is for all $\gamma\in(0,1]$}
The main term in the bound of Corollary \ref{cor:dynamocstrconvex} scales as $T^{\frac{\dc + \gamma}{\dc+2\gamma}}$. In \cite{li2019dimension}, this rate was obtained for a different algorithm, under the additional assumption that for each $\bc\in[0,1]^{\dc}$ the global minimizer $\argmin_{\bx\in\mathbb{R}^{\dx}}f(\bx,\bc)$ is an interior point of the constraint set $\com$. While this assumption might be hard to check in practice, it does simplify the analysis considerably.
Note also that the multiplicative factors in Corollary \ref{cor:dynamocstrconvex} exhibit a very mild dependence on both dimensions $p$ and $d$, never exceeding $p^{1/2}d$ in the main term.

%In the next section, we demonstrate that the regret achieved by Algorithm \ref{algo2}, as outlined in Corollary \ref{cor:dynamocstrconvex}, is optimal as a function of $T$, up to logarithmic factor. In Algorithm \ref{algo2}, we suppose knowledge of $T$. However, employing the well-known doubling trick can render the algorithm adaptive to $T$.

\section{Lower bound}\label{sec:lower}
%In this section, we provide a lower bound in Theorem \ref{thm:lw_STRc_cont} indicating that the imposed Assumption \ref{ass:Lip} is necessary in order to achieve a sub-linear dynamic regret. In other words, we demonstrate that there exists no randomized procedure that achieves a sub-linear dynamic regret over the whole class of functions which fail to satisfy Assumption \ref{ass:Lip}. In the next definition, with a slight abuse of notation, we generalize the definition of $\mathcal{F}_{\gamma}(\Lc)$ to the case when $\gamma = 0 $, which characterize discontinuous functions.

%In this section, we observe that Assumption \ref{ass:Lip} is necessary in order to achieve a sub-linear dynamic regret. Specifically, we present a lower bound demonstrating that there exists no randomized procedure achieving a sub-linear dynamic regret over the whole class of functions which fail to satisfy Assumption \ref{ass:Lip}. In addition, the lower bound shows that Algorithm \ref{algo2} is optimal, up to logarithmic factors concerning the horizon parameter $T$. 

In this section, along with the classes $\mathcal{F}_{\gamma}(\Lc)$, $\gamma\in(0,1]$, we will consider the class $\mathcal{F}_{0}(\Lc)$, which includes discontinuous functions.
\begin{definition}
For $\Lc>0$, let $\mathcal{F}_{0}(\Lc)$ be the class of functions such that $f\in\mathcal{F}_{0}(\Lc)$ if 
\begin{align}\label{eq:goodalgo2}
    |f(\bx,\bc) - f(\bx,\bc')|\leq \Lc,\quad\text{for all}\quad \bx\in\com, \,\,\bc,\bc'\in[0,1]^{\dc}.
    \end{align}
\end{definition}
We will assume that the noises $\{\xi_t\}_{t=1}^{T}$ are independent with cumulative distribution function $F$ satisfying the condition
\begin{align}\label{con:lw}
    \int\log\left(\d F\left(u\right)/d F\left(u+v\right)\right) \d F(u)\leq I_0v^2,\quad\quad |v|<v_0,
\end{align}
for some $0 < I_0 < \infty, 0 < v_0 \leq \infty.$ It can be shown that condition \eqref{con:lw} is satisfied for $F$ with a sufficiently smooth density and finite Fisher information. If $F$ follows a Gaussian distribution, \eqref{con:lw} holds with $v_0 = \infty$. Note that Gaussian noise satisfies Assumption \ref{ass:noise} used to prove the upper bound.

Next, we introduce the distribution of the contexts used in deriving the lower bound. Similar to previous sections, for any $\numbin \in \mathbb{N}$, consider the partition of $[0,1]^{\dc}$, denoted by $\{B_{i}\}_{i=1}^{\numbin^{\dc}}$, for which we assume that $B_{i}$'s have equal volumes. Moreover, for $i \in [\numbin^{\dc}]$, let $\mathbf{P}_{i}$ denote the uniform distribution on $G_i = (1/2)(B_i + \bb_i)$, where $\bb_i$ is the barycenter of $B_i$. Then, we assume that $\{\bc_t\}_{t=1}^{T}$ are independently distributed according to $\mathbb{P}_{\numbin}$, where 
\begin{align}\label{eq:badcontexbehave}
   \mathbb{P}_{\numbin}(\Upsilon) = \numbin^{-\dc}\sum_{i=1}^{\numbin^{\dc}}\mathbf{P}_i(\Upsilon\cap G_i), 
\end{align}
for $\Upsilon\subseteq [0,1]^{\dc}$. Distribution $\mathbb{P}_{\numbin}$ represents a challenging scenario for the learning process. Indeed, the contexts are uniformly spread, preventing concentration in any specific region of $[0,1]^{\dc}$. Moreover, they are not clustered near the boundaries of each cell, thus preventing the learner from exploiting proximity between the cells. See Appendix \ref{app:3} for further details.

\begin{theorem}\label{thm:lw_STRc_cont}
For $(\alpha,\Lx,M, \gamma,\Lc)\in (0,\infty)\times [3\alpha,\infty)\times [\alpha+1,\infty)\times[0,1]\times [0,\infty)$ consider the class of functions $\bar{\mathcal{F}} = \mathcal{F}_{\alpha,\Lx}(M)\cap\mathcal{F}_{\gamma}(\Lc)$. Let $\com = B^d$,
%$\{\bx\in\mathbb{R}^{\dx}:\norm{\bx}\leq 1\}$, 
and let $\{\bc_t\}_{t=1}^{T}$ be independently distributed according to $\mathbb{P}_{\numbin}$, where $\numbin = \max\left(1,\floor{\left(\min\left(1,\Lc^2\right)T\right)^{
\frac{1}{\dc+2\gamma}}}\right)$. Let $\pi\in\Pi$ be any randomized procedure and assume that $\{\bz_t\}_{t=1}^{T}$ are outputs of $\pi$. Then, 
\begin{align}\label{eq:lwBound}
\sup_{f\in\bar{\mathcal{F}}}\sum_{t=1}^{T}\Exp\left[f(\bz_t,\bc_t
    ) - \min_{\bz\in\com}f(\bz,\bc_t)\right]\geq \cst\left(\min\left(1, \Lc^{\frac{2(\dc+\gamma)}{\dc+2\gamma}}\right)T^{\frac{\dc+\gamma}{\dc+2\gamma}}+\min\left(T,\dx \sqrt{T}\right)\right),
\end{align}
where $\cst>0$ is a constant that does not depend on $\dx$, $\dc$, and $T$. 
\end{theorem}
Applying \eqref{eq:lwBound} with $\Lc>0$ and 
$\gamma = 0$ shows that no randomized procedure can achieve sub-linear dynamic regret on the corresponding class 
$\bar{\mathcal{F}}$. Thus, controlling the increments of $f$ with any $L>0$, in the absence of continuity with respect to $\bc$, is not sufficient to get a sub-linear dynamic regret.

For 
$\gamma\in (0,1]$, the bound \eqref{eq:lwBound} combined with Corollary \ref{cor:dynamocstrconvex} implies that Algorithm \ref{algo2} with $K$ chosen as in Corollary \ref{cor:dynamocstrconvex} achieves the optimal rate of regret in $T$ up to logarithmic factors. An additional (mild) gap between the upper and lower bounds is due to the discrepancy in the factors depending on $\dx,
\dc,$ and~$\Lc$.

If $\Lc = 0$ the bound in \eqref{eq:lwBound} scales as $\min(T,d\sqrt{T})$. In this case, the objective functions are constant with respect to the context variable, $f(\bx,\bc_1) = f(\bx,\bc_t)$, for $\bx\in\mathbb{R}^{\dx}$ and $t\in[T]$, so that the contextual regret in \eqref{eq:lwBound} is equal to the static regret. Moreover, by convexity, both regrets are bounded from below by the optimization error corresponding to $f(\cdot,\bc_1)$ multiplied by $T$. Using this remark we notice that, for $\Lc = 0$, the bound \eqref{eq:lwBound} agrees with the lower bound on the optimization error for strongly convex and smooth functions, known to be of the order $d/\sqrt{T}$, see  \cite{shamir2013complexity,akhavan2020exploiting,akhavan2023gradient}.  

Furthermore, note that since the class of strongly convex functions is included in the class of convex functions, \eqref{eq:lwBound} is a valid lower bound for convex contextual bandits. Together with Corollary~\ref{cor:conv}, this implies that the minimax optimal rate as a function of $T$ for convex contextual bandits is of the order $T^{\frac{\dc+\gamma}{\dc+2\gamma}}$ up to logarithmic factors.
 
Theorem \ref{thm:lw_STRc_cont} can be stated in a more general manner. In the definition of $\Pi$, the randomized procedures are restricted to query only within the constraint set $\com$. However, in the proof of the lower bound, we did not impose such a restriction, allowing $\pi$ to query anywhere within $\mathbb{R}^{\dx}$. Thus, Theorem \ref{thm:lw_STRc_cont} remains valid for this broader set of algorithms. 
Furthermore, the proof of Theorem \ref{thm:lw_STRc_cont} can be extended to hold for any convex body $\com$ and not only for $\com = B^d$.  Indeed, by shifting and scaling the unit Euclidean ball, we can embed it in $\com$. Subsequently, we can apply the same shifting and scaling to the action argument of the set of functions $f_{\bomega, \btau}$ in the proof of Theorem \ref{thm:lw_STRc_cont}.

\section*{Disclosure of Funding}
The research of Arya Akhavan was funded by UK Research and Innovation (UKRI) under the UK government’s Horizon Europe funding guarantee [grant number EP/Y028333/1] and in part by EU Project ELSA under grant agreement No. 101070617.

\bibliography{bibliography}

\appendix

\section*{Appendix}
In Appendix \ref{app:1}, we present the proof of Theorem \ref{thm:general}. {Appendices \ref{app:2} and \ref{app:3} begin with a set of auxiliary and preliminary lemmas, followed by the proofs of the main results: Theorem \ref{thm:statstrconv}, Corollary \ref{cor:dynamocstrconvex}, and Theorem \ref{thm:lw_STRc_cont}.} In Appendix \ref{app:self-concord}, we collect some properties of self-concordant barrier functions.

\section{Proof of Theorem \ref{thm:general}}\label{app:1}
\begin{proof}[Proof of Theorem \ref{thm:general}]
   For $i\in [\numbin^{\dc}]$, let $N_i(T)$ the number of $\bc_t$'s belonging to $B_i$. Denote $i_1,\dots,i_{N_i(T)}$ as the set of all indices such that $\bc_{i_j}\in B_i$ for all $j\in [N_i(T)]$. We have
\begin{align*}
    \sum_{t=1}^{T}\left(f(\bz_t,\bc_t) - \min_{\bz\in\com}f(\bz,\bc_t)\right)=\sum_{i=1}^{\numbin^{\dc}}\sum_{j=1}^{N_{i}(T)}\left(f(\bz_{i_j},\bc_{i_j}) - \min_{\bz\in\com} f(\bz,\bc_{i_j})\right).
\end{align*}
For $i\in [\numbin^{\dc}]$, let $\bb_i\in B_i$ be the barycenter of the cube $B_i$, and $\bx^*_i\in \argmin_{\bx\in\com} f(\bx,\bb_i)$. Since \eqref{eq:goodalgox} 
holds it follows that
    \begin{align*}
        \Exp\left[\sum_{j=1}^{N_{i}(T)}\left(f(\bz_{i_j},\bc_{i_j}) - f(\bx_i^*,\bc_{i_j})\right)|N_i(T)\right]\leq F(N_i(T)) {F_1(N_i(T)) \le 
        F(N_i(T)) F_1(T)},
    \end{align*}
where $F$ is a concave function and {$F_1$ is a non-decreasing function}. Note that $\Exp\left[N_i(T)\right] = p_iT$, for some $p_i\in[0,1]$, such that $\sum_{i=1}^{\numbin^{\dc}} p_i = 1$. By taking expectations from both sides and applying Jensen's inequality we get
\begin{align*}
\Exp\left[\sum_{j=1}^{N_i(T)}\left(f(\bz_{i_j},\bc_{i_j})-f(\bx^*_i,\bc_{i_j})\right)\right]&\leq F(p_iT){F_1(T)}.
\end{align*}
Summing both sides over $i\in[K^{\dc}]$ and using Jensen's inequality again we obtain
\begin{align}\nonumber
\Exp\left[\sum_{i=1}^{\numbin^{\dc}}\sum_{j=1}^{N_i(T)}\left(f(\bz_{i_j},\bc_{i_j})-f(\bx^*_i,\bc_{i_j})\right)\right]
&\leq \sum_{i=1}^{\numbin^{\dc}}F(p_iT){F_1(T)}
\\
&\leq \nonumber
\numbin^{\dc}F\left(\sum_{i=1}^{\numbin^{\dc}}\frac{p_iT}{\numbin^{\dc}}\right){F_1(T)} 
\\
\label{eq:thmgeneral0}
&= \numbin^{\dc}F\left(\frac{T}{\numbin^{\dc}}\right){F_1(T)}.
\end{align}
If, in addition, $\pi\in \Pi(\mathcal{F},\tau_1,\tau_2,\tau_3,\Tau_0)$, then for $T\geq \numbin^{\dc}\,\Tau
_0$ we get that
\begin{align}\label{eq:thmgeneral1}
\Exp\left[\sum_{i=1}^{\numbin^{\dc}}\sum_{j=1}^{N_i(T)}\left(f(\bz_{i_j},\bc_{i_j})-f(\bx^*_i,\bc_{i_j})\right)\right]\leq C\dx^{\tau_1}T^{\tau_2}\log^{\tau_3}(T+1)\numbin^{\dc(1-\tau_2)}.
\end{align}
Let $\bz^*_{i_{j}} \in \argmin_{\bz\in\com} f(\bz, \bc_{i_{j}})$. We can write
\begin{align*}
\Exp\left[\sum_{t=1}^{T}\left(f(\bz_{t},\bc_{t})-\min_{\bz\in\com}f(\bz,\bc_t)\right)\right]&=\Exp\left[\sum_{i=1}^{\numbin^{\dc}}\sum_{j=1}^{N_i(T)}\left(f(\bz_{i_j},\bc_{i_j})-f(\bx^*_i,\bc_{i_j})+f(\bx^*_i,\bc_{i_j}) - f(\bz^*_{i_{j}},\bc_{i_j})\right)\right]
\\&=\Exp\bigg[\sum_{i=1}^{\numbin^{\dc}}\sum_{j=1}^{N_i(T)}\bigg(f(\bz_{i_j},\bc_{i_j})-f(\bx^*_i,\bc_{i_j})+f(\bx^*_i,\bc_{i_j}) - f(\bx_{i}^{*},\bb_i) \\&\quad+ f(\bx_{i}^{*},\bb_{i}) -f(\bz^{*}_{i_{j}},\bc_{i_j})\bigg)\bigg].
\end{align*}
Since $f(\bx_{i}^{*},\bb_{i}) \leq f(\bz_{i_{j}}^{*},\bb_{i})$, we have
\begin{align*}
\Exp\left[\sum_{t=1}^{T}\left(f(\bz_{t},\bc_{t})-\min_{\bz\in\com}f(\bz,\bc_t)\right)\right]&\leq \Exp\bigg[\sum_{i=1}^{\numbin^{\dc}}\sum_{j=1}^{N_i(T)}\bigg(f(\bz_{i_j},\bc_{i_j})-f(\bx^*_i,\bc_{i_j})+ f(\bx_{i}^{*},\bc_{i_{j}}) - f(\bx_{i}^{*},\bb_i) \\&\quad+ f(\bz^{*}_{i_{j}},\bb_i) - f(\bz^{*}_{i_{j}},\bc_{i_{j}}) \bigg)\bigg]
\\&\leq \Exp\left[\sum_{i=1}^{\numbin^{\dc}}\sum_{j=1}^{N_i(T)}\left(f(\bz_{i_j},\bc_{i_j})-f(\bx^*_i,\bc_{i_j})+2L\norm{\bc_{i_{j}} -\bb_i}^{\gamma}\right)\right],
\end{align*}
where the last inequality is due to Assumption \ref{ass:Lip}. {Let $\delta = \max_{i\in[\numbin^{\dc}]}\max_{\bx,\by\in B_i}\norm{\bx-\by}$. Since all $B_i$'s are hypercubes with edges of length $1/K$ we have $\delta = \sqrt{p}/\numbin$,} so that
\begin{align}\label{eq:thmgeneral2}
\Exp\left[\sum_{t=1}^{T}\left(f(\bz_{t},\bc_{t})-\min_{\bz\in\com}f(\bz,\bc_t)\right)\right]&\leq \Exp\left[\sum_{i=1}^{\numbin^{\dc}}\sum_{j=1}^{N_i(T)}\left(f(\bz_{i_j},\bc_{i_j})-f(\bx^*_i,\bc_{i_j})\right)\right]+2TL\left(\frac{\sqrt{\dc}}{\numbin}\right)^{\gamma}.
\end{align}
{
Combining \eqref{eq:thmgeneral0} and \eqref{eq:thmgeneral2}, gives
\begin{align*}
\Exp\left[\sum_{t=1}^{T}\left(f(\bz_{t},\bc_{t})-\min_{\bz\in\com}f(\bz,\bc_t)\right)\right]&\leq \numbin^{\dc}F\left(\frac{T}{\numbin^{\dc}}\right)+2TL\left(\frac{\sqrt{\dc}}{\numbin}\right)^{\gamma},
\end{align*}
which concludes the proof of \eqref{eq:eboundgenX}. In order to prove the bound \eqref{eq:generaldynamo}, we note that, by \eqref{eq:thmgeneral1} and \eqref{eq:thmgeneral2}, 
\begin{align*}
\Exp\left[\sum_{t=1}^{T}\left(f(\bz_{t},\bc_{t})-\min_{\bz\in\com}f(\bz,\bc_t)\right)\right]&\le C\dx^{\tau_1}T^{\tau_2}\log^{\tau_3}(T+1)\numbin^{\dc(1-\tau_2)}+2TL\left(\frac{\sqrt{\dc}}{\numbin}\right)^{\gamma}.
\end{align*}
The two terms on the right hand side are of the same order of magnitude if $\numbin = \numbin(\tau_1,\tau_2,\tau_3, \gamma)$. Under this choice of $K$ we get
\begin{align*}
\Exp\bigg[\sum_{t=1}^{T}\bigg(f(\bz_{t},\bc_{t})-\min_{\bz\in\com}&f(\bz,\bc_t)\bigg)\bigg]\le C'\max\bigg(\dx^{\tau_1}T^{\tau_2}\log^{\tau_3}(T+1),
\\&
\left(\Lc \dc^{\frac{\gamma}{2}}\right)^{\frac{\dc(1-\tau_2)}{\dc(1-\tau_2)+\gamma}}\left(\dx^{\tau_1}\log^{\tau_3}(T+1)\right)^{\frac{\gamma}{\dc(1-\tau_2)+\gamma}}T^{\frac{(\dc-\gamma)(1-\tau_2) + \gamma}{\dc(1-\tau_2)+\gamma}}\bigg),
\end{align*}
where $C'>0$ is a constant that does not depend on $T, p$ and $d$.}
\end{proof}

\section{Proof of Theorem \ref{thm:statstrconv} and Corollary \ref{cor:dynamocstrconvex}}\label{app:2}
\begin{lemma}\label{lem:step}
    Let $\{\eta_t\}_{t=1}^T$ be as in Algorithm \ref{algoBCO}. Then, for all $t\geq 2$ we have
    \begin{align}
        1-\frac{\eta_t}{\eta_{t-1}}\leq \mathbf{1}\left(t\geq 16\mu\right)\sqrt{\frac{1}{t}}.
    \end{align}
  
\end{lemma}
\begin{proof}
Note that $t/\log(t+1)$ is an increasing function for all $t\geq 1$. Let $t_0$ be the smallest positive integer such that $t_0/\log(t_0+1)\geq 16(\mu + \Lx/\alpha)$, and note that 
\begin{align*}
    \eta_t = (4\dx q_T)^{-1}, \quad \text{for }t< t_0,
\end{align*}
and 
\begin{align*}
    \eta_t = (4\dx q_T)^{-1}\left(\frac{16(\mu + \Lx/\alpha)\log(t+1)}{t}\right)^{\frac{1}{2}}, \quad \text{for }t\geq t_0,
\end{align*}
where $q_T=M +2\sigma\sqrt{\log(T+1)}$. If $t <t_0$ then $1-\eta_{t}/\eta_{t-1} = 0$. For the case $t = t_0$, we can write

\begin{align*}
    1- \frac{\eta_t}{\eta_{t-1}} = \frac{\sqrt{t_0} - \sqrt{16(\mu+\Lx/\alpha)\log(t_0+1)}}{\sqrt{t_0}}. 
\end{align*}
Note that $t_0 \leq 16\left(\mu+\Lx/\alpha\right)\log(t_0+1) + 1$. Thus,
\begin{align*}
    1- \frac{\eta_t}{\eta_{t-1}} \leq \frac{\sqrt{t_0}-\sqrt{t_0-1}}{\sqrt{t_0}} \leq \sqrt{\frac{1}{t_0}} = \sqrt{\frac{1}{t}}. 
\end{align*}
For $t > t_0$ we have
\begin{align*}
    1-\frac{\eta_t}{\eta_{t-1}} = 1-\sqrt{\frac{t-1}{t}}\sqrt{\frac{\log(t+1)}{\log(t)}}\leq \frac{\sqrt{t} - \sqrt{t-1}}{\sqrt{t}} \leq \sqrt{\frac{1}{t}}.
\end{align*}
By combing the above inequalities we get
\begin{align*}
        1-\frac{\eta_t}{\eta_{t-1}}\leq \mathbf{1}\left(t\geq t_0\right)\sqrt{\frac{1}{t}}.
\end{align*}
Since $16(\mu+\Lx/\alpha) \geq 16$, we have $t_0\geq 2$, and consequently $\log(t_0+1)\geq 1$. Thus,
\begin{align*}
    t_0 \geq 16(\mu+\frac{\Lx}{\alpha})\log(t_0+1)  \geq 16\mu.
\end{align*}
Therefore, for all $t\in[T]$ we have 
\begin{align*}
    \mathbf{1}\left(t\geq t_0\right)\leq  \mathbf{1}\left(t\geq 16\mu\right),
\end{align*}
which concludes the proof.
\end{proof}
Before proceeding with the proofs, let us recall the concept of the Newton decrement of a self-concordant function. Let $\Psi:\text{int}(\com)\to\mathbb{R}$ be a self-concordant function. 
The Newton decrement of $\Psi$ at $\bx\in\text{int}(\com)$ is defined as
\begin{align}\label{eq:NDec}
    \lambda(\Psi,\bx) = \max_{\bh\in\mathbb{R}^{\dx}}\{\nabla \Psi(\bx)[\bh]|\nabla^2 \Psi(\bx)[\bh,\bh]\leq 1\}.
\end{align}

In the case that $\nabla^2 \Psi(\bx)$ is positive definite, for all $\bx\in\text{int}(\com)$ \eqref{eq:NDec} can be written as
\begin{align}\label{eq:ND}
    \lambda(\Psi, \bx) = \sqrt{(\nabla \Psi(\bx))^{\top}(\nabla^2 \Psi(\bx))^{-1}\nabla \Psi(\bx)}.
\end{align}
The proof of the equivalency of \eqref{eq:NDec} and \eqref{eq:ND} can be found in \cite[Chapter 2, Section IVa]{nemirovski2004interior}.
Moreover, if $\Psi$ is a $\mu$-self-concordant barrier then $\lambda(\Psi, \bx) \leq \mu^{1/2}$ for all $\bx\in \text{int}(\com)$. 

The next lemma plays a key role in our analysis. 

\begin{lemma}\label{lem:goldencon}
    Let $\Psi:\text{int}(\com)\to\mathbb{R}$ be a self-concordant function. Then $\Psi$ attains its minimum on $\text{int}(\com)$ if and only if there exists $\bx\in \text{int}(\com)$ with $\lambda(\Psi, \bx)< 1$. Moreover, if $\lambda(\Psi, \bx)< 1$, then
    \begin{align*}
        ((\bx - \bx^*)^{\top}\nabla^2 \Psi(\bx)(\bx - \bx^*))^{\frac{1}{2}}\leq \frac{\lambda(\Psi,\bx)}{1-\lambda(\Psi,\bx)},
    \end{align*}
where $\bx^* = \argmin_{\bx\in\text{int}(\com)}\Psi(\bx)$.
\end{lemma}
For a proof of Lemma \ref{lem:goldencon}, we refer the reader to \cite[Chapter 2, Section VIII]{nemirovski2004interior}.

\begin{lemma}\label{lem:mainopt1}
Define the event $$\Lambda_T = \{|\xi_t|\leq 2\sigma\sqrt{\log(T+1)}\, \text{ for all }t\in[T]\}.$$ Let $\{\bx_t\}_{t=0}^{T}$ be the iterates defined in {Algorithm \ref{algoBCO}}. If $\Lambda_T$ holds, for all $t\in[T]$ we have
\begin{align}\label{eq:claim}
   \norm{\bx_{t-1}-\bx_t}_t \leq 2\left(\eta_t\norm{\bg_t}_t^{*} + \sqrt{\frac{\mu}{t}}\right),
\end{align}
where we use the notation $\norm{\bu}_t: = \sqrt{\bu^{\top}P_t^{-2}\bu}$ and $\norm{\bu}_t^{*}: = \sqrt{\bu^{\top}P_t^{2}\bu}$ \, for all $\bu\in\mathbb{R}^{\dx}$. 
\end{lemma}
\begin{proof}
Let $\tilde{\Phi}_t:\text{int}(\com)\to\mathbb{R}$ be such that
\begin{align*}
    \tilde{\Phi}_t(\bu) = \eta_t\left(\sum_{k=1}^{t}\langle\bg_k,\bu\rangle + \frac{\alpha}{2}\sum_{k=1}^{t}\norm{\bu - \bx_{k-1}}^2\right)+\mathcal{R}(\bu).
    \end{align*}
Note that $\tilde{\Phi}_t$ is a self-concordant function with $\bx_{t}\in\argmin_{\bx\in\text{int}(\com)}\tilde{\Phi}_t(\bx)$, and 
$\nabla^2\tilde{\Phi}_{t}(\bx_{t-1}) = P_t^{-2}$. Consider the Newton decrement $\lambda(\tilde{\Phi}_t,\bx_{t-1})$, cf. \eqref{eq:ND}. 
The first part of the proof is dedicated to providing an upper bound for $\lambda(\tilde{\Phi}_t,\bx_{t-1})$. Then we use Lemma \ref{lem:goldencon} to prove \eqref{eq:claim}. For $t\in[T]$ we have
\begin{align}\label{eq:NDBound}
    \lambda(\tilde{\Phi}_t,\bx_{t-1}) = \sqrt{(\nabla \tilde{\Phi}_t(\bx_{t-1}))^{\top}(\nabla^2 \tilde{\Phi}_t(\bx_{t-1}))^{-1}\nabla \tilde{\Phi}_t(\bx_{t-1})} = \norm{\nabla \tilde{\Phi}_t(\bx_{t-1})}_t^{*}.
\end{align}

For $t = 1$, we can write 
\begin{align}\label{eq:NDBound0}
   \norm{\nabla\tilde{\Phi}_1(\bx_0)}_1^{*} = \norm{\eta_1\langle \bg_1,\bx_0\rangle + \nabla\mathcal{R}(\bx_0)} \leq \eta_1\norm{\bg}_1^{*},
\end{align}
where the last inequality follows from the facts that $\bx_0 = \argmin_{\bx\in\text{int}(\com)}\mathcal{R}(\bx)$ and $\mathcal{R}(\bx_0) = 0$.

Moreover, for $t\geq 2$ we have

\begin{align}\nonumber
    \nabla \tilde{\Phi}_t(\bx_{t-1}) &= \frac{\eta_t}{\eta_{t-1}}\left(\eta_{t-1}\sum_{k=1}^{t-1}\bg_k+\eta_{t-1}\alpha\sum_{k=1}^{t-1}(\bx-\bx_{k-1})+ \nabla\mathcal{R}(\bx_{t-1})\right) +  \eta_t\bg_t \\\nonumber&\phantom{00000}+\left(1-\frac{\eta_t}{\eta_{t-1}}\right)\nabla \mathcal{R}(\bx_{t-1}) 
    \\\nonumber&=\frac{\eta_t}{\eta_{t-1}}\nabla\tilde{\Phi}_{t-1}(\bx_{t-1}) + \eta_t \bg_t +\left(1-\frac{\eta_t}{\eta_{t-1}}\right)\nabla \mathcal{R}(\bx_{t-1}) \\\label{eq:NDBound1}&= \eta_t \bg_t +\left(1-\frac{\eta_t}{\eta_{t-1}}\right)\nabla \mathcal{R}(\bx_{t-1}) .
\end{align}

Using %\eqref{eq:NDBound1} in 
\eqref{eq:NDBound} --  \eqref{eq:NDBound0} we find that, for all $t\in[T]$,
\begin{align*}
    \lambda(\tilde{\Phi}_t,\bx_{t-1}) = \norm{\nabla\tilde{\Phi}_t(\bx_{t-1})}_t^{*}\leq \eta_t\norm{\bg_t}_t^* + \left(1-\frac{\eta_t}{\eta_{t-1}}\right)\norm{\nabla \mathcal{R}(\bx_{t-1})}_t^{*}.
\end{align*}
Thanks to Lemma \ref{lem:step} we can bound the term $1-\eta_{t}/\eta_{t-1}$ to obtain
\begin{align*}
     \lambda(\tilde{\Phi}_t,\bx_{t-1}) 
  \leq \eta_t\norm{\bg_t}_t^{*} + \norm{\nabla \mathcal{R}(\bx_{t-1})}_t^{*}\mathbf{1}\left(t\geq 16\mu\right)\sqrt{\frac{1}{t}}.
\end{align*}

On the other hand, since $\mathcal{R}$ is a $\mu$-self-concordant function we have
\begin{align*}
    (\norm{\nabla \mathcal{R}(\bx_{t-1})}_t^{*})^{2} \leq  \nabla \mathcal{R}(\bx_{t-1})^{\top}(\nabla^{2}\mathcal{R}(\bx_{t-1}))^{-1}\nabla \mathcal{R}(\bx_{t-1})=\lambda^2(\mathcal{R},\bx_{t-1})\leq \mu, 
\end{align*}
so that 
\begin{align}\label{eq:NDBound2}
    \lambda(\tilde{\Phi}_t,\bx_{t-1})
\leq\eta_t\norm{\bg_t}_t^{*}+\mathbf{1}\left(t\geq 16\mu\right)\sqrt{\frac{\mu}{t}}\leq \eta_t\norm{\bg_t}_t^{*} + \frac{1}{4}.
\end{align}
Since we assumed that $\Lambda_T$ holds, we can write
\begin{align}\label{eq:varbound}
    \norm{\bg_t}_t^{*} = \norm{\dx(f(\bz_t,\bc_t)+\xi_t)P_t^{-1}\bzeta_t}_{t}^{*} \leq \dx(|f(\bz_{t},\bc_t)|+|\xi_t|)\leq \dx(M+2\sigma\sqrt{\log(T+1)}).
\end{align}
By the definition of $\eta_t$, for all $t\in[T]$ we have that $\eta_t \leq (4\dx(M+2\sigma\sqrt{\log(T+1)}))^{-1}$ which implies that $\eta_t\norm{\bg}_t^{*}\leq 1/4$. Using this inequality in \eqref{eq:NDBound2} gives $\lambda(\tilde{\Phi}_t,\bx_{t-1})\leq 1/2$. Therefore, by Lemma \ref{lem:goldencon} we can write
\begin{align*}
    \norm{\bx_{t-1} - \bx_{t}}_t = \norm{\bx_{t-1} - \argmin_{\bx\in\text{int}(\com)}\tilde{\Phi}_t(\bx)}_t \leq 2\lambda(\tilde{\Phi}_t,\bx_{t-1}). 
\end{align*}
The last display together with \eqref{eq:NDBound2} yields
\begin{align*}
    \norm{\bx_{t-1}-\bx_t} &\leq 2\left(\eta_t\norm{\bg_t}_t^{*} + \mathbf{1}\left(t\geq 16\mu\right)\sqrt{\frac{\mu}{t}}\right)
    \leq 2\left(\eta_t\norm{\bg_t}_t^{*} + \sqrt{\frac{\mu}{t}}\right).
\end{align*}

\end{proof}

\begin{lemma}\label{lem:mainopt}
      Consider a function $h:\mathbb{R}^{\dx}\times [0,1]^{\dc}\to\mathbb{R}$, and set $h_t(\bx) = h(\bx,\bc_t)$ for $t\geq 1$. Assume that $\{h_t\}_{t=1}^{T}$ is a sequence of $\alpha$-strongly convex functions for some $\alpha>0$. Let $\{\bx_t\}_{t=0}^{T}$ be the iterates  defined in {Algorithm \ref{algoBCO}}.  Let $\bg_t$'s be random vectors in $\mathbb{R}^d$ such that $\Exp\left[\bg_t|\bx_{t-1},\bc_t,\Lambda_T\right] = \nabla h_t(\bx_{t-1})$ for $t\geq 1$. Then, for any $\bx\in\text{int}(\com)$ we have
   \begin{align*}
     \sum_{t=1}^{T}\Exp\left[h_t(\bx_{t-1})-h_t(\bx)|\Lambda_T\right] 
     \leq 2\dx q_T\left(\dx q_T\sum_{t=1}^{T}\eta_t + 2\sqrt{\mu T}\right)+\eta_T^{-1}\mathcal{R}(\bx),
    \end{align*}
where $q_T=M +2\sigma\sqrt{\log(T+1)}$.
\end{lemma}
\begin{proof}
Fix $\bx\in\text{int}(\com)$. Since $h_t$ is a strongly convex function, for any $\bx\in\com$ we can write
\begin{align}\label{eq:eq:mainopt-1}
    \sum_{t=1}^{T}\left(h_t(\bx_{t-1}) - h_t(\bx)\right) \leq \sum_{t=1}^{T}\langle\Exp\left[\bg_t|\bx_{t-1},\bc_t,\Lambda_T\right],\bx_{t-1}-\bx\rangle - \frac{\alpha}{2}\sum_{t=1}^{T}\norm{\bx-\bx_{t-1}}^2.
\end{align}
To simplify the notation, we rewrite $\bx_t$ as follows
\begin{align*}
    \bx_t &= \argmin_{\bu\in\com}\sum_{k=1}^{t}\langle\bg_k,\bu\rangle + \underbrace{\frac{\alpha}{2}\sum_{k=1}^{t}\norm{\bu-\bx_{k-1}}^2 + \eta_{t}^{-1}\mathcal{R}(\bu)}_{\defeq R_t(\bu)}
    \\& = \argmin_{\bu\in\com}\underbrace{\sum_{k=1}^{t}\langle\bg_k,\bu\rangle + R_t(\bu)}_{\defeq\Phi_t(\bu)} = \argmin_{\bu\in\com}{\Phi_t(\bu)}.
\end{align*}
Moreover, set $\Phi_0(\bu) = \mathcal{R}(\bu)$ for all $\bu\in \text{int}(\com)$. Since $\bx_0=\argmin_{\bu\in\text{int}(\com)}\Phi_0(\bu)$ and we assume that $\mathcal{R}(\bx_0)=0$, we have
\begin{align*}
    -\sum_{t=1}^{T}\langle \bg_t,\bx\rangle &= -\Phi_T(\bx) + \Phi_T(\bx_T)  - \Phi_T(\bx_T) +R_{T}(\bx) 
    \\&\leq  \sum_{t=1}^{T}\left(\Phi_{t-1}(\bx_{t-1})-\Phi_t(\bx_t)\right) +R_{T}(\bx),
\end{align*}
where the last display is due to the fact that $\Phi_T(\bx_T)\leq \Phi_T(\bx)$. Subsequently, 
\begin{equation}
\begin{aligned}\label{eq:mainopt0}
   \sum_{t=1}^{T}\langle\bg_t,\bx_{t-1}-\bx\rangle \leq \sum_{t=1}^{T}\left(\underbrace{\Phi_{t-1}(\bx_{t-1})-\Phi_t(\bx_t)+\langle\bg_t,\bx_{t-1}\rangle}_{\text{term I}}\right) +R_{T}(\bx).
\end{aligned}
\end{equation}
In order to control term I, we first derive an alternative expression for $\Phi_t$. By the definition of $\Phi_t$, 
\begin{align*}
    \Phi_t(\bx_t) = \sum_{k=1}^{t}\langle\bg_k,\bx_t\rangle + R_{t}(\bx_t)&= \sum_{k=1}^{t-1}\langle\bg_k,\bx_t\rangle + \langle\bg_t,\bx_t\rangle+R_{t-1}(\bx_t) + \left(R_{t}(\bx_t) -R_{t-1}(\bx_t)\right)
    \\&= \Phi_{t-1}(\bx_t) + \langle\bg_t,\bx_t\rangle+\left(R_{t}(\bx_t) -R_{t-1}(\bx_t)\right),
\end{align*}
so that
\begin{align}\nonumber
    \text{term I} &= \left(\Phi_{t-1}(\bx_{t-1}) + \langle\bg_t,\bx_{t-1}\rangle\right) - \left(\Phi_{t-1}(\bx_t) + \langle\bg_t,\bx_t\rangle\right) +R_{t-1}(\bx_t) - R_{t}(\bx_t)
    \\\label{eq:mainopt1}&\leq \langle\bg_t,\bx_{t-1} - \bx_t\rangle,
\end{align}
where the last inequality is due to the fact that $R_{t-1}(\bx_t)\leq R_t(\bx_t)$, and $\Phi_{t-1}(\bx_{t-1})=\min_{\bz\in\com}\Phi_{t-1}(\bz)\leq \Phi_{t-1}(\bx_t)$. From \eqref{eq:mainopt0} and \eqref{eq:mainopt1} we obtain
\begin{align*}
     \sum_{t=1}^{T}\langle\bg_t,\bx_{t-1}-\bx\rangle  \leq \sum_{t=1}^{T}\langle\bg_t,\bx_{t-1}-\bx_t\rangle + \frac{\alpha}{2}\sum_{t=1}^{T}\norm{\bx-\bx_{t-1}}^2  +\eta_T^{-1}\mathcal{R}(\bx).
\end{align*}
 Taking expectations conditioned on the event $\Lambda_T$ gives
\begin{align*}
     \Exp\left[\sum_{t=1}^{T}\langle\bg_t,\bx_{t-1}-\bx\rangle |\Lambda_T\right]
     &\leq \Exp\left[\sum_{t=1}^{T}\langle\bg_t,\bx_{t-1}-\bx_t\rangle + \frac{\alpha}{2}\sum_{t=1}^{T}\norm{\bx-\bx_{t-1}}^2  |\Lambda_T\right]+\eta_T^{-1}\mathcal{R}(\bx)
     \\&\leq \Exp\left[\sum_{t=1}^{T}\norm{\bg_t}_t^{*}\norm{\bx_{t-1}-\bx_t}_t + \frac{\alpha}{2}\sum_{t=1}^{T}\norm{\bx-\bx_{t-1}}^2  |\Lambda_T\right]+\eta_T^{-1}\mathcal{R}(\bx).
\end{align*}
%where the last inequality is derived by Cauchy-Schwarz inequality for dual norms. 
Furthermore, by the tower rule, 
\begin{align*}
\Exp\left[\sum_{t=1}^{T}\langle\Exp\left[\bg_t|\bx_{t-1},\bc_t,\Lambda_T\right],\bx_{t-1}-\bx\rangle\big|\Lambda_T\right] 
&\leq \Exp\left[\sum_{t=1}^{T}\norm{\bg_t}_t^{*}\norm{\bx_{t-1}-\bx_t}_t+ \frac{\alpha}{2}\sum_{t=1}^{T}\norm{\bx-\bx_{t-1}}^2  |\Lambda_T\right]
\\&\phantom{0000000}+\eta_T^{-1}\mathcal{R}(\bx).
\end{align*}
Using the assumption that $\Exp\left[\bg_t|\,\bx_{t-1},\bc_t,\Lambda_T\right] = \nabla h_t(\bx_{t-1})$, the fact that $\{h_t\}_{t=1}^{T}$ are $\alpha$-strongly convex functions, and \eqref{eq:eq:mainopt-1}, yields
\begin{align}\label{eq:mainoptX}
    \Exp\left[\sum_{t=1}^{T}(h_t(\bx_{t-1})-h_t(\bx)) |\Lambda_T\right]\leq\Exp\left[\sum_{t=1}^{T}\norm{\bg_t}_t^{*}\norm{\bx_{t-1}-\bx_t}_t|\Lambda_T\right]+\eta_T^{-1}\mathcal{R}(\bx).
\end{align}

We now apply Lemma \ref{lem:mainopt1} to bound the value  $\norm{\bx_{t-1}-\bx_t}_t$ on the event $\Lambda_T$. It follows that 
 \begin{align}
     \Exp\left[\sum_{t=1}^{T}(h_t(\bx_{t-1})-h_t(\bx)) |\Lambda_T\right]\leq 2\left(\sum_{t=1}^{T}\eta_t\Exp\left[(\norm{\bg_t}_t^{*})^{2}|\Lambda_T\right] + \sqrt{\frac{\mu}{t}}\Exp\left[\norm{\bg_t}_t^{*}|\Lambda_T\right]\right)+\eta_T^{-1}\mathcal{R}(\bx).
    \end{align}
Recall that $q_T = M+2\sigma\sqrt{\log(T+1})$, and under $\Lambda_T$, the uniform bound \eqref{eq:varbound} gives $\norm{\bg_t}_t^{*}\leq \dx q_T$. Hence,

 \begin{align*}
     \sum_{t=1}^{T}\Exp\left[h_t(\bx_{t-1})-h_t(\bx)|\Lambda_T\right] &\leq 2\left(\dx^2q_T^2\sum_{t=1}^{T}\eta_t + \dx q_T\sum_{t=1}^{T}\sqrt{\frac{\mu}{t}}\right)+\eta_T^{-1}\mathcal{R}(\bx)
     \\&\leq 2\dx q_T\left(\dx q_T\sum_{t=1}^{T}\eta_t + 2\sqrt{\mu T}\right)+\eta_T^{-1}\mathcal{R}(\bx).
    \end{align*}
\end{proof}

\begin{lemma}\cite[Corollary 6]{hazan2014bandit}\label{lem:sur}
   Let $A\in \mathbb{R}^{\dx\times \dx}$ be an invertible matrix. Let $\bU$ be distributed uniformly on $B^{\dx}$. Let ${\sf f}_{A}(\bx,\bc) \eqdef \Exp\left[f(\bx + A\bU,\bc)\right]$ for $\bx\in\mathbb{R}^{\dx}$ and $\bc\in[0,1]^{\dc}$. Then,
    \begin{align*}
        \nabla_{\bx} {\sf f}_{A}(\bx,\bc) = \Exp\left[\dx f(\bx + A\bzeta,\bc)A^{-1}\bzeta\right],
    \end{align*}
    where $\bzeta$ is uniformly distributed on $\partial B^{\dx}$. If Assumption \ref{ass:Hstrconvex}(i) holds, then ${\sf f}_{A}(\cdot,\bc)$ is $\alpha$-strongly convex on $\mathbb{R}^{\dx}$ for all $\bc\in[0,1]^{\dc}$. 
\end{lemma}

\begin{lemma}\label{lem:newsur}
Let Assumption \ref{ass:noise}(ii) hold. Let $\bU$ be distributed uniformly on $B^{\dx}$. For $t\geq 1$, let ${\sf f}_{t,P}:\mathbb{R}^\dx\times [0,1]^{\dc}\to\mathbb{R}$ be a surrogate function defined by the relation
    \begin{align}\label{eq:surrog}
       {\sf f}_{t,P}(\bx,\bc)=\Exp\left[f(\bx+P_t\bU,\bc)|\bx_{t-1}\right]. 
    \end{align}
Let $\{\bg_t\}_{t=1}^{T}$ and $\{\bx_t\}_{t=0}^{T-1}$ be defined as in {Algorithm \ref{algoBCO}}, and $\Lambda$ be any event that only depends on $\{\xi_t\}_{t=1}^{T}$. Then for all $t\in[T]$ we have
\begin{align*}
    \Exp\left[\bg_t|\bx_{t-1},\bc_t,\Lambda\right] = \nabla_{\bx} {\sf f}_{t,P}(\bx_{t-1},\bc_t).
\end{align*}
\end{lemma}
\begin{proof}
    By the definition of $\bg_t$ we have 
    \begin{align*}
        \Exp\left[\bg_t|\bx_{t-1},\bc_t,\Lambda\right] &= \dx\Exp\left[\left(f(\bx_{t-1}+P_t\bzeta_t,\bc_t)+\xi_t\right)P_t^{-1}\bzeta_t|\bx_{t-1},\bc_t,\Lambda\right]
        \\=&\dx\Exp\left[f(\bx_{t-1}+P_t\bzeta_t,\bc_t)P_t^{-1}\bzeta_t|\bx_{t-1},\bc_t\right] + \underbrace{\dx\Exp[\xi_t|\bx_{t-1},\bc_t,\Lambda]P_t^{-1}\Exp\left[\bzeta_t\right]}_{=0},
    \end{align*}
where the last equality is due to Assumption \ref{ass:noise}(ii). Now, by Lemma \ref{lem:sur} we have 
\begin{align*}
    \dx\Exp\left[f(\bx_{t-1}+P_t\bzeta_t,\bc_t)P_t^{-1}\bzeta_t|\bx_{t-1},\bc_t\right] = \nabla_{\bx} {\sf f}_{t,P}(\bx_{t-1},\bc_t),
\end{align*}
which concludes the proof.
\end{proof}
\begin{lemma}\cite[Lemma 4]{hazan2014bandit}\label{lem:conbound} For $\mu>0$ let $\mathcal{R}:\text{int}(\com)\to\mathcal{R}$ be a $\mu$-self-concordant barrier with $\min_{\bx\in\text{int}(\com)}\mathcal{R}(\bx) =0$. Then, for all $\bx,\by\in\text{int}(\com)$ we have
\begin{align*}
    \mathcal{R}(\bx)-\mathcal{R}(\by) \leq \mu \log\left(\frac{1}{1-\pi_{\bx}(\by)}\right),
\end{align*}
where $\pi_{\bx}(\by) = \inf\{t\geq 0 : \bx + t^{-1}(\by-\bx)\}$ is the Minkowsky function. Particularly, for $\bx \in \text{int}(\com)$ if $\bx_0 \in \argmin_{\bu\in\text{int}(\com)}\mathcal{R}(\bu)$, and $\bx' = (1-1/T)\bx + (1/T)\bx_0$ then 
\begin{align}\label{eq:concordbound}
    \mathcal{R}(\bx')\leq \mu \log\left(T\right).
\end{align}
\end{lemma}

\begin{proof}[Proof of Theorem \ref{thm:statstrconv}]
Let $\bz' = (1-1/T)\bz + (1/T)\bx_0$. Note that
\begin{align}\nonumber
    \sum_{t=1}^{T}\Exp\left[f(\bz_t,\bc_t) - f(\bz,\bc_t)\right] &=\sum_{t=1}^{T}\Exp\left[f(\bz_t,\bc_t) - f(\bz',\bc_t)\right] +\sum_{t=1}^{T}\Exp\left[f(\bz',\bc_t)-f(\bz,\bc_t) \right]
\\\label{eq:mT1}&\leq\sum_{t=1}^{T}\Exp\left[f(\bz_t,\bc_t) - f(\bz',\bc_t)\right]  + \frac{1}{T}\sum_{t=1}^{T}\Exp\left[f(\bx_0,\bc_t) - f(\bz,\bc_t)\right]
\\\label{eq:X1}&\leq\sum_{t=1}^{T}\Exp\left[f(\bz_t,\bc_t) - f(\bz',\bc_t)\right]  + 2M,
\end{align}
where \eqref{eq:mT1} is due to convexity of $f(\cdot,\bc)$. Invoking the event $\Lambda_T$ we can write
\begin{align*}
    \sum_{t=1}^{T}\Exp\left[f(\bz_t,\bc_t) - f(\bz',\bc_t)\right] &=    \sum_{t=1}^{T}\Exp\left[f(\bz_t,\bc_t) - f(\bz',\bc_t)|\Lambda_T^c\right]\mathbf{P}[\Lambda_T^c]+  \sum_{t=1}^{T}\Exp\left[f(\bz_t,\bc_t) - f(\bz',\bc_t)|\Lambda_T\right]\mathbf{P}[\Lambda_T]
  .
\end{align*}
Since $\{\xi_t\}_{t=1}^{T}$ are $\sigma$-sub-Gaussian we have
\begin{align*}
    \mathbf{P}\left[\Lambda_T^{c}\right]  \leq \sum_{t=1}^{T}\mathbf{P}\left[|\xi_t|>2\sigma\sqrt{\log(T+1})\right]\leq 2\sum_{t=2}^{T+1}T^{-2} = 2T^{-1}.
\end{align*}
Thus, 
\begin{align}\label{eq:X2}
    \sum_{t=1}^{T}\Exp\left[f(\bz_t,\bc_t) - f(\bz',\bc_t)\right] \leq 4M+\sum_{t=1}^{T}\Exp\left[f(\bz_t,\bc_t) - f(\bz',\bc_t)|\Lambda_T\right].
\end{align}
Next, introducing the surrogate function ${\sf f}_{t,P}(\bx,\bc) = \Exp\left[f(\bx + P_t\bU,\bc)|\bx_{t-1}\right]$ we consider the decomposition

\begin{align*}
    &\sum_{t=1}^{T}\Exp\left[f(\bz_t,\bc_t) - f(\bz',\bc_t)|\Lambda_T\right] \\
&\hspace{1cm}=\underbrace{\sum_{t=1}^{T}\Exp\left[f(\bz_t,\bc_t) - f(\bx_{t-1},\bc_t)|\Lambda_T\right]}_{\text{term I}} + \underbrace{\sum_{t=1}^{T}\Exp\left[f(\bx_{t-1},\bc_t) - {\sf f}_{t,P}(\bx_{t-1},\bc_t)|\Lambda_T\right]}_{\text{term II}} \\&\hspace{2cm}+ \underbrace{\sum_{t=1}^{T}\Exp\left[{\sf f}_{t,P}(\bz',\bc_t) - f(\bz',\bc_t)|\Lambda_T\right]}_{\text{term III}}+\underbrace{\sum_{t=1}^{T}\Exp\left[{\sf f}_{t,P}(\bx_{t-1},\bc_t)- {\sf f}_{t,P}(\bz',\bc_t)|\Lambda_T\right]}_{\text{term IV}}  .
\end{align*}
In what follows, we bound each term in this decomosition separately. For term I, by Assumption \ref{ass:Hstrconvex}(ii) we have
\begin{align*}
    \text{term I} = \Exp\left[\sum_{t=1}^{T}\Exp\left[f(\bz_t,\bc_t)-f(\bx_{t-1},\bc_t)|\bx_{t-1},\bc_t\right]|\Lambda_T\right] &\leq \frac{\Lx}{2}\sum_{t=1}^{T}\Exp\left[\norm{P_t\bzeta_t}^2\right]
    \\&\leq \frac{\Lx}{2}\sum_{t=1}^{T}\Exp\left[\norm{P_t}^2_{\text{op}}\right]\leq \frac{\Lx}{2\alpha}\sum_{t=1}^{T}\frac{1}{\eta_t t},
\end{align*}
where $\norm{\cdot}_{\text{op}}$ denotes the operator norm.
By the fact that $f(\cdot,\bc_t)$ is a convex function and by Jensen's inequality, we deduce that term II is negative. For term III, by Assumption \ref{ass:Hstrconvex}(ii) we have
\begin{align*}
    \text{term III}&\leq \frac{\Lx}{2}\sum_{i=1}^{T}\Exp\left[\norm{P_t\bU}^2\right]
    \leq \frac{\Lx}{2}\sum_{t=1}^{T}\Exp\left[\norm{P_t}^2_{\text{op}}\right]\leq \frac{\Lx}{2\alpha }\sum_{t=1}^{T}\frac{1}{\eta_t t}.
\end{align*}
To control term IV, first note that by Lemma \ref{lem:sur} $\{{\sf f}_{t,P}(\cdot,\bc_t)\}_{t=1}^{T}$ are $\alpha$-strongly convex functions. Furthermore, by Lemma \ref{lem:newsur} we have that $\Exp\left[\bg_t|\bx_{t-1},\bc_{t},\Lambda_T\right] = \nabla {\sf f}_{t,P}(\bx_{t-1},\bc_t)$. Therefore, by Lemma \ref{lem:mainopt} we obtain 
\begin{align*}
    \text{term IV} &\leq 2\dx^2q_T^2\sum_{t=1}^{T}\eta_t + 4\dx q_T\sqrt{\mu T}+\eta_T^{-1}\mathcal{R}(\bz'),
    \end{align*}
where $q_T=M +2\sigma\sqrt{\log(T+1)}$ and $\bz' = (1-1/T)\bz + (1/T)\bx_0$. By Lemma \ref{lem:conbound} we have $\mathcal{R}(\bz')\leq \mu\log(T)$, so that
 \begin{align*}
    \text{term IV} &\leq 2\dx^2q_T^2\sum_{t=1}^{T}\eta_t + 4\dx q_T\sqrt{\mu T}+\eta_T^{-1}\mu\log(T).
    \end{align*}
Thus,
\begin{align*}
     \sum_{t=1}^{T}\Exp\left[f(\bz_t,\bc_t) - f(\bz',\bc_t)|\Lambda_T\right] \leq 2\dx^2q_T^2\sum_{t=1}^{T}\eta_t + 4\dx q_T\sqrt{\mu T}+\eta_T^{-1}\mu\log(T) +\frac{\Lx}{\alpha}\sum_{t=1}^{T}\frac{1}{\eta_t t}.
\end{align*}
Since $\eta_t = (4\dx q_T)^{-1}\min\big(1,\frac{\nu\log(t+1)}{t}\big)^{\frac{1}{2}}$ with $\nu = 16(\mu + \Lx/\alpha)$ we get
\begin{align*}
     \sum_{t=1}^{T}\Exp\left[f(\bz_t,\bc_t) - f(\bz',\bc_t)|\Lambda_T\right] &\leq \dx q_T\sqrt{\nu \log(T+1)T} + 4\dx q_T\sqrt{\mu T}+
     \\&+4\dx q_T\mu\max\left(\log^{2}(T+1),\frac{T\log(T+1)}{\nu}\right)^{\frac{1}{2}} +\frac{\Lx}{\alpha}\sum_{t=1}^{T}\frac{1}{\eta_t t}.
\end{align*} 
Using the fact that $16\mu\leq \nu$, and $\max\{a,b\}\leq a+b$ for $a,b>0$, we further simplify the above bound to obtain
\begin{align*}
     \sum_{t=1}^{T}\Exp\left[f(\bz_t,\bc_t) - f(\bz',\bc_t)|\Lambda_T\right] &\leq \frac{9}{4}\dx q_T\sqrt{\nu \log(T+1)T} 
     +4\dx q_{T}\mu\log(T+1)+\frac{\Lx}{\alpha}\sum_{t=1}^{T}\frac{1}{\eta_t t}.
\end{align*} 
Since  $t/\log(t+1)\geq \sqrt{t}$ for all $t\in[T]$, then if $t\geq \nu^2$ we have $t\geq \nu\log(t+1)$. Thus, 
\begin{align*}
    \frac{\Lx}{\alpha}\sum_{t=1}^{T}\frac{1}{\eta_t t} &\leq \frac{4\Lx}{\alpha}\dx q_T\left(\sum_{t=1}^{\nu^2}\frac{1}{t} + \sum_{t=2}^{T}\sqrt{\frac{1}{\nu t\log(t+1)}} \right)
    \\&\leq 4\dx q_T\left(\frac{\Lx}{\alpha}\log(\nu^2+1)+\frac{2\Lx}{\alpha}\sqrt{\frac{T}{\nu}}\right)
    \leq 4\dx q_T\left(\frac{\Lx}{\alpha}\log(\nu^2+1)+\frac{1}{8}\sqrt{\nu T}\right),
\end{align*}
where the last inequality is due to the fact that $16\Lx/\alpha \leq \nu$. Thus,
\begin{align*}
     \sum_{t=1}^{T}\Exp\left[f(\bz_t,\bc_t) - f(\bz',\bc_t)|\Lambda_T\right] &\leq \dx q_T\left(\frac{11}{4}\sqrt{\nu \log(T+1)T} 
     +4\left(\mu\log(T+1)+\frac{\Lx}{\alpha}\log(\nu^2+1)\right)\right)
     \\&\leq \frac{\dx q_T}{4}\left(11\sqrt{\nu \log(T+1)T} 
     +\nu\left(\log(T+1)+\log(\nu^2+1)\right)\right).
\end{align*} 
Since $q_T = M+2\sigma\sqrt{\log(T+1)}$ we have
\begin{equation}\label{eq:X3}
\begin{aligned}
     \sum_{t=1}^{T}\Exp\left[f(\bz_t,\bc_t) - f(\bz',\bc_t)|\Lambda_T\right] &\leq  \frac{\dx}{4}\big(M+2\sigma\sqrt{\log(T+1)}\big)\bigg(11\sqrt{\nu T \log(T+1)} \\&\quad+\nu\left(\log(T+1) + 2\log(\nu+1)\right)\bigg).
\end{aligned} 
\end{equation}
{
Combining \eqref{eq:X1}, \eqref{eq:X2}, and \eqref{eq:X3} implies
\begin{align*}
     \sum_{t=1}^{T}\Exp\left[f(\bz_t,\bc_t) - f(\bz,\bc_t)|\Lambda_T\right] &\leq  \frac{\dx}{4}\big(M+2\sigma\sqrt{\log(T+1)}\big)\bigg(11\sqrt{\nu T \log(T+1)} \\&\quad+\nu\left(\log(T+1) + 2\log(\nu+1)\right)\bigg) + 6M.
\end{align*}
Note that since $\nu\geq 16$ and $2\log(\nu +1) \geq 5.5$, we have
\begin{align*}
     \sum_{t=1}^{T}\Exp\left[f(\bz_t,\bc_t) - f(\bz,\bc_t)|\Lambda_T\right] &\leq  \frac{\dx}{4}\big(5.3M+2\sigma\sqrt{\log(T+1)}\big)\bigg(11\sqrt{\nu T \log(T+1)} \\&\quad+\nu\left(\log(T+1) + 2\log(\nu+1)\right)\bigg).
\end{align*}
Thus, we can write
\begin{align*}
    \sum_{t=1}^{T}\Exp\left[f(\bz_t,\bc_t) - f(\bz,\bc_t)\right] &\leq F(T)F_1(T), %\frac{\dx}{4}\big(M+2\sigma\sqrt{\log(T+1)}\big)\bigg(11\sqrt{\nu T \log(T+1)} \\&\quad+\nu\left(\log(T+1) + 2\log(\nu+1)\right)\bigg)+ 6M .
\end{align*}
where
\begin{align*}
    F(T) = 11\sqrt{\nu T \log(T+1)}+\nu\left(\log(T+1) + 2\log(\nu+1)\right),
\end{align*}

and
\begin{align*}
        F_{1}(T)=\frac{\dx}{4}\left(5.3M+2\sigma\sqrt{\log(T+1)}\right).
\end{align*}
Noticing that 
$F$ is a concave and $F_1$ is a non-decreasing function we complete the first part of the proof. For the second part of the proof, assume that $T \geq 16(\mu+\Lx/\alpha)^2$. In this case we have $T\geq \nu$ since $\Lx/\alpha\ge 1$. Therefore, 
\begin{align*}
    F(T) \leq 11\sqrt{\nu T \log(T+1)} + 3\nu\log(T+1) .
\end{align*}
On the other hand, the condition $T \geq 16(\mu+\Lx/\alpha)^2$ implies that
\begin{align*}
    \sqrt{\nu T \log(T+1)} \geq \frac{\nu}{2} T^{\frac{1}{4}}\sqrt{\log(T+1)} \geq \frac{\nu}{2}  \log(T+1), 
\end{align*}
and therefore, 
\begin{align*}
    F(T) &\leq 17\sqrt{\nu T \log(T+1)} = 68\sqrt{\left(\mu + \frac{\beta}{\alpha}\right) T \log(T+1)}.
\end{align*}
}
\end{proof}

\begin{proof}[Proof of Corollary \ref{cor:dynamocstrconvex}]
Since $F$ in \eqref{eq:concaveboundstr} is a concave function  
we can apply \eqref{eq:eboundgenX} to obtain
\begin{align*}
\Exp\big[\sum_{t=1}^{T}\big(f(\bz_t,\bc_t) - \min_{\bz\in\com}f(\bz,\bc_t)\big)\big]
&\leq \numbin^pF\left(\frac{T}{\numbin^p}\right)F_1(T) + 2LT\left(\frac{\sqrt{\dc}}{\numbin}\right)^{\gamma}.
\end{align*}
Using the definition of $F$ and the inequality $\mu \leq \cst \dx^{\dxmu}$, we get
\begin{align}\label{eq:cor1-bound}
  \Exp\big[\sum_{t=1}^{T}\big(f(\bz_t,\bc_t) - \min_{\bz\in\com}f(\bz,\bc_t)\big)\big]
&\leq C_*\Big( L_1\dx^{1+\dxmu}\numbin^{\dc} + L_2\sqrt{\numbin^{\dc}\dx^{2+\dxmu} T}  + LT\left(\frac{\sqrt{\dc}}{K}\right)^{\gamma}\Big),
\end{align}
where $L_1=\sqrt{\log (T+1)}\log (Td+1)$, $L_2=\log(T+1)$, and $C_*>0$ is a constant depending only on $M,\sigma, \Lx/\alpha$.

%where $\lesssim$ is the sign of inequality to within a factor, which is logarithmic in $T$ and $d$.
Consider now two cases. First, let $\Lc \dc^{\frac{\gamma}{2}}\dx^{-(1+\dxmu/2)} T^{1/2}\log^{-1}(T+1)>1$. In this case, since $x/2\le \lfloor x\rfloor \le x$ for $x> 1$, we obtain that $$\numbin=\numbin(1+\dxmu/2,1/2,1, \gamma)\asymp \Big(\Lc \dc^{\frac{\gamma}{2}}\dx^{-(1+\dxmu/2)} T^{1/2}\log^{-1}(T+1)\Big)^{\frac{2}{\dc+2\gamma}},$$ 
so that, to within a constant depending only on $M,\sigma, \Lx/\alpha$,  the right hand side of \eqref{eq:cor1-bound} is of the order
\begin{align*}
   L_1\dx^{1+\dxmu} \Big(\Lc \dc^{\frac{\gamma}{2}}\dx^{-(1+\dxmu/2)}L_2^{-1}\Big)^{\frac{2\dc}{\dc+2\gamma}} T^{\frac{\dc}{\dc+2\gamma}}
   +
 L_2^{\frac{2\gamma}{\dc+2\gamma}} \Big(\Lc \dc^{\frac{\gamma}{2}}\Big)^{\frac{\dc}{\dc+2\gamma}}\dx^{\frac{(2+\dxmu)\gamma}{\dc+2\gamma}}T^{\frac{\dc + \gamma}{\dc+2\gamma}} .
\end{align*}
In the second case, $\Lc \dc^{\frac{\gamma}{2}}\dx^{-(1+\dxmu/2)} T^{1/2} \log^{-1}(T+1)\le 1$. Then $\numbin(1+\dxmu/2,1/2,1,\gamma)=1$, so that the right hand side of \eqref{eq:cor1-bound} does not exceed
\begin{align*}
 C_*(L_1\dx^{1+\dxmu} + L_2\sqrt{\dx^{2+\dxmu} T} + (\Lc \dc^{\frac{\gamma}{2}} T^{1/2}) \sqrt{T})
 \le 
 C_*(L_1\dx^{1+\dxmu} + (L_2+ 1)\sqrt{\dx^{2+\dxmu} T}). 
\end{align*}
Putting these remarks together we get
\begin{align*}
 &   \Exp\big[\sum_{t=1}^{T}\big(f(\bz_t,\bc_t) - \min_{\bz\in\com}f(\bz,\bc_t)\big)\big]
\le C \max\bigg\{L_1\dx^{1+\dxmu} + L_2\sqrt{\dx^{2+\dxmu} T}, \, 
    \\&
    \qquad \qquad \qquad \qquad \qquad 
    L_1\Big(\Lc\dc^{\frac{\gamma}{2}}L_2^{-1}\Big)^{\frac{2\dc}{\dc+2\gamma}}
    \dx^{\frac{2\gamma(1+\dxmu)-\dc}{\dc+2\gamma}}
    T^{\frac{\dc}{\dc+2\gamma}},
    \,
    L_2^\frac{2\gamma}{\dc+2\gamma}\Big(\Lc \dc^{\frac{\gamma}{2}}\Big)^{\frac{\dc}{\dc+2\gamma}}
    \dx^{\frac{(2+\dxmu)\gamma}{\dc+2\gamma}}
    T^{\frac{\dc + \gamma}{\dc+2\gamma}}  
    \bigg\},
\end{align*}
where $C>0$ is a constant depending only on $M,\sigma, \Lx/\alpha$.
\end{proof}

\section{Proof of Theorem \ref{thm:lw_STRc_cont}}\label{app:3}
\begin{lemma}\label{lem:con} Suppose that $\mathbf{P}_1$ and $\mathbf{P}_2$ are two probability measures that are defined on a probability space $(X,
\mathcal{M})$, such that $\mathbf{P}_1$ is absolutely continuous with respect to $\mathbf{P}_2$. Let $\Upsilon\in \mathcal{M}$ be such that $\mathbf{P}_1(\Upsilon) = \mathbf{P}_2(\Upsilon)\neq 0$. Let $\tilde{\mathbf{P}}_1$ and $\tilde{\mathbf{P}}_2$ be the conditional probability measures of $\mathbf{P}_1$ and $\mathbf{P}_2$ given the event $\Upsilon$, respectively, i.e., for $i = 1, 2$ and $\Gamma \in \mathcal{M}$
\begin{align*}
\tilde{\mathbf{P}}_{i}(\Gamma) = \mathbf{P}_{i}(\Gamma\cap \Upsilon) \mathbf{P}_{i}^{-1}(\Upsilon).
\end{align*}
Then,
\begin{enumerate}
    \item For $i=1,2$
    $$\frac{\d\tilde{\mathbf{P}}_{i}(\bx)}{\d\mathbf{P}_{i}(\bx)} = \mathbb{1}(\bx\in \Upsilon)\mathbf{P}_{i}^{-1}(\Upsilon).$$
    \item $KL(\tilde{\mathbf{P}}_{1},\tilde{\mathbf{P}}_{2})=\int \log\left(\frac{\d\mathbf{P}_{1}(\bx)}{\d\mathbf{P}_{2}(\bx)}\right)\mathbb{1}(\bx\in \Upsilon)\mathbf{P}_{1}^{-1}(\Upsilon)\d\mathbf{P}_{1}(\bx).$
\end{enumerate}
\end{lemma}

\begin{proof}[Proof of Theorem \ref{thm:lw_STRc_cont}]
We choose a finite set of functions in a way that its members cannot be distinguished from each other with positive probability but are separated enough from each other to guarantee that the maximal regret for this family is not smaller than the desired lower bound.

Set $\numbin=\max\left(1,\floor{\left(\min\left(1,\Lc^2\right)T\right)^{
\frac{1}{\dc+2\gamma}}}\right)$.
Let $(B_i)_{i=1}^{\numbin^{\dc}}$ be a partition of $[0,1]^{\dc}$, where $B_i$'s are disjoint cubes with edges of length~$1/\numbin$. Then, for any $\bc\in[0,1]^{\dc}$ there exists $i\in\{1,\dots, \numbin^{\dc}\}$ that satisfies $d(\bc,\partial B_i)\leq \delta$, where $\partial B_i$ is the boundary of $B_i$, $d(\bc,\partial B_i) = \min_{\bu\in\partial B_i}\norm{\bc-\bu}$  and $\delta = \frac{1}{2\numbin}$,  cf. \eqref{eq1-lem-proj}.

Let $\eta_0:\mathbb{R}\to\mathbb{R}$ be an infinitely many times differentiable function such that
\begin{align*}
    \eta_0(x) = \begin{cases}
    =1 \quad\quad &\text{if }|x|\leq 1/4,\\
    \in (0,1)\quad\quad &\text{if }1/4<|x|< 1,\\
    =0 \quad\quad &\text{if }|x|\geq 1.
    \end{cases}
\end{align*}
Set $\eta(x) = \int_{-\infty}^{x}\eta_0(u)\d u$. Let $\Omega = \{-1,1\}^{\numbin^{\dc}}$ and $\numhypo$ be the sets of binary sequences of length $\numbin^{\dc}$ and $\dx-1$, respectively. Consider the finite set of functions $f_{\bomega,\btau}:\mathbb{R}^{\dx}\times [0,1]^{\dc}\to\mathbb{R}$, $\bomega \in \Omega$, $\btau\in\numhypo$ such that:
\begin{equation}\label{lw:function}
\begin{aligned}
    f_{\bomega,\btau}(\bx,\bc) = \alpha\norm{\bx}^2+r_1\Lc\eta\Big(x_1\delta^{-\frac{\gamma}{2}}\Big)&\left(\sum
_{j=1}^{\numbin^{\dc}}\omega_jd^{\gamma}(\bc, \partial B_j)\mathbb{1}\left(
\bc\in B_j\right)\right)\\
&+ r_2h^{2}\left(\sum_{i=2}^{\dx}\tau_i\eta\left(\frac{x_i}{h}\right)\right),\quad\quad \bx = (x_1,\dots,x_{\dx}),
\end{aligned}
\end{equation}
where $\omega_j,\tau_i \in \{-1,1\}$, $h = \min\left(\dx^{-\frac{1}{2}},T^{-\frac{1}{4}}\right)$ and $r_1,r_2$ are fixed positive numbers that will be chosen small enough.
By Lemma \ref{lem:family} we have that $f_{\bomega,\btau}\in \mathcal{F}_{\alpha,\Lx}(M)\cap \mathcal{F}_{\gamma}(\Lc)$ provided $r_1\leq \min\left(1, 1/\Lc^2\right)\min\left(1/2\eta(1), \alpha/L'\right), r_2\leq \min\left(1/2\eta(1), \alpha/L'\right)$, where $L' = \max_{x\in\mathbb{R}}|\eta''\left(x\right)|$. Furthermore, by choosing $r_1\leq \alpha/\left(2\max\left(1,\Lc\right)\right)$ and $r_2\leq \alpha/2$ we have $$\alpha^{-1}\delta^{-\gamma}r_1\Lc\sum
_{j=1}^{\numbin^{\dc}}\omega_j d^{\gamma}(\bc_t,\partial B_j)\mathbb{1}\left(
    \bc\in B_j\right)\leq \frac{1}{2},\quad\quad\text{}\quad\quad \alpha^{-1}r_2 \leq \frac{1}{2},$$ and that under this condition the equation $\nabla_{\bx}f_{\bomega,\btau}(\bx,\bc) = 0$ has the solution
\begin{align*}
    \bx^{*}(\bomega,\btau,\bc) = \left(x^{*}_{1}(\bomega,\btau,\bc),\dots, x^{*}_{\dx}(\bomega,\btau,\bc)\right),
\end{align*}
where $x^{*}_{1}(\bomega,\btau,\bc) = -(1/2)\alpha^{-1}\delta^{-\frac{\gamma}{2}}r_1\Lc\sum
_{j=1}^{\numbin^{\dc}}\omega_j d^{\gamma}(\bc_t,\partial B_j)\mathbb{1}\left(
\bc\in B_j\right)$ and $x^{*}_{i}(\bomega,\btau,\bc) = -(1/2)\tau_ir_2\alpha^{-1}h$, for $2\leq i\leq \dx$.
Furthermore,
\begin{align*}
    \norm{\bx^{*}(\bomega,\btau,\bc)}^2\leq \frac{1}{4\alpha^2}\left(r_1^2\Lc^2\delta^{\gamma}+ \dx r_2^2h^2\right)\leq \frac{1}{8},
\end{align*}
which implies $\bx^{*}(\bomega,\btau,\bc)\in \com$.

For any fixed $\bomega\in\Omega,\btau\in\numhypo$, and $1\leq t\leq T$, we denote by $\mathbf{P}_{\bomega, \btau,t}$ the probability measure corresponding to the joint distribution of $(\bz_1, y_1^{t},\bc_1^t)$ where $y_k = f_{\bomega,\btau}(\bz_k,\bc_k)+\xi_k$ with independent identically distributed $\xi_k$’s such that $\eqref{con:lw}$ holds and $\bz_k$’s chosen by the randomized procedure $\pi$. We have
\begin{align}\label{lw:dist}
    \d\mathbf{P}_{\bomega,\btau, t}\left(\bz_1, y_1^{t},\bc_1^t\right) = \d F\left(y_1-f_{\bomega,\btau}\left(\bz_1,
    \bc_1\right)\right)\d \mathbb{P}_{\numbin}(\bc_1)\prod_{i=2}^{t}\d F\left(y_i-f_{\bomega,\btau}\left(\Phi_i\left(\bz_1, y_1^{i-1},\bc_1^{i}\right),
    \bc_i\right)\right)\d \mathbb{P}_{\numbin}(\bc_i).
\end{align}
Without loss of generality, we omit here the dependence of $\Phi_i$ on $\bz_2, \dots, \bz_{i-1}$ since $\bz_i, i \geq 2$, is a Borel function of $\bz_1,y_1,\dots,y_{i-1},\bc_1,\dots, \bc_i$. Let $\Exp_{\bomega,\btau,t}$ denote the expectation w.r.t. $\mathbf{P}_{\bomega,\btau,t}$.

Note that
\begin{equation}\label{eq:strongc}
  \sum_{t=1}^{T}\Exp_{\bomega,\btau,t}\left[f_{\bomega,\btau}(\bz_t,\bc_t) - \min_{\bx\in\com}f_{\bomega,\btau}(\bx,\bc_t)\right] \ge \frac{\alpha}{2} \sum_{t=1}^{T}\Exp_{\bomega,\btau,t}\left[\norm{\bz_t-\bx^{*}(\bomega,\btau,\bc_t)}^2\right]  
\end{equation}
due to the uniform strong convexity of functions $\bx\mapsto f_{\bomega,\btau}(\bx,\bc)$.

Consider the statistics 
\begin{align*}
    (\hat{\bomega}_t, \hat{\btau}_t)\in \argmin_{\bomega\in \Omega,\btau\in\numhypo}\norm{\bz_t-\bx^{*}(\bomega,\btau, \bc_t)}.
\end{align*}
Since $\norm{\bx^{*}(\hat{\bomega}_t,\hat{\btau}_t,\bc_t) -\bx^{*}(\bomega,\btau,\bc_t)}\leq\norm{\bz_t - \bx^{*}(\hat{\bomega}_t,\hat{\btau}_t,\bc_t)}+\norm{\bz_t-\bx^{*}(\bomega,\btau,\bc_t)}\leq 2\norm{\bz_t-\bx^{*}(\bomega,\btau,\bc_t)}$ for all $\bomega\in\Omega,\btau\in\numhypo$ we obtain
\begin{align*}
    \Exp_{\bomega,\btau,t}\left[\norm{\bz_t-\bx^{*}(\bomega,\btau,\bc_t)}^{2}\right]&\geq \frac{1}{4}\Exp_{\bomega,\btau,t}\left[\norm{\bx^{*}(\hat{\bomega}_t,\hat{\btau}_t,\bc_t)-\bx^{*}(\bomega,\btau,\bc_t)}^{2}\right]\\
    &=\frac{1}{4}\alpha^{-2}\Lc^2\delta^{-\gamma}r_1^2\Exp_{\bomega,\btau,t}\left[\sum_{j=1}^{\numbin^{\dc}}d^{2\gamma}(\bc_t, \partial B_j)\mathbb{1}\left(\hat{\omega}_{t,j}\neq\omega_j
    \right)\mathbb{1}\left(
    \bc_t\in B_j\right)\right]
\\&\phantom{00000}+\frac{1}{4}\alpha^{-2}h^{2}r_2^2\Exp_{\bomega,\btau,t}\left[\rho(\btau,\hat{\btau_t})\right],
\end{align*}
where $\rho(\btau,\hat{\btau}_t) = \sum_{i=2}^{\dx}\mathbb{1}\left(\btau_i\neq\hat{\btau}_{t,i}\right)$ is the Hamming distance between $\btau$ and $\hat{\btau}_t$, with $\hat{\bomega}_t=\left(\hat{\omega}_{t,1},\dots,\hat{\omega}_{t,\numbin^{\dc}}\right)$ and $ \hat{\btau}_t=\left(\hat{\tau}_{t,2},\dots,\hat{\tau}_{t,\dx}\right)$. 
Note that, due to \eqref{eq2-lem-proj}, for $\varepsilon=1/2$ we have
$$
\mathbb{P}_{\numbin}\left[d(\bc_t, \partial B_j)\geq \frac{1}{4\numbin}\Big |\bc_t\in B_j\right] \geq 
\mathbb{P}_{\numbin}\left[\norm{\bc_t - \bb_j}_{\infty} \leq \frac{1}{4\numbin}\Big |\bc_t\in B_j\right] = 1,$$
which yields
$$\mathbb{P}_{\numbin}\left[d^{2\gamma}(\bc_t, \partial B_j)\geq \left(\frac{\delta^2}{4\dc^2}\right)^{\gamma}\Big|\bc_t\in B_j\right] = 1.$$

%Note that $\mathbf{P}\left[d^2(\bc_t, \partial B_j)\geq \frac{\delta^2}{d^2}|\bc_t\in B_j\right] \geq 1/4$
Thus,
\begin{align*}
    \Exp_{\bomega,\btau,t}&\left[\norm{\bz_t-\bx^{*}(\bomega,\bc_t)}^{2}\right]\geq
    \frac{1}{4}\alpha^{-2}\delta^{-\gamma} r_1^2\Lc^2\sum_{j=1}^{\numbin^{\dc}}\Exp_{\bomega,\btau,t}\left[d^{2\gamma}(\bc_t, \partial B_j)\mathbb{1}\left(\hat{\omega}_{t,j}\neq\omega_j
    \right)|\bc_t\in B_j\right]\mathbb{P}_{\numbin}\left(\bc_t\in B_j\right)
\\&\phantom{1000000000000}+\frac{1}{4}\alpha^{-2}r_2^2h^{2}\Exp_{\bomega,\btau,t}\left[\rho(\btau,\hat{\btau_t})\right]
\\&= \frac{1}{4}\alpha^{-2}r_1^2\Lc^2\delta^{-\gamma} \sum_{j=1}^{\numbin^{\dc}}\bigg(\Exp_{\bomega,\btau,t}\left[d^{2\gamma}(\bc_t, \partial B_j)\mathbb{1}\left(\hat{\omega}_{t,j}\neq\omega_j
    \right)\Big|\bc_t\in B_j,d^{2\gamma}(\bc_t, \partial B_j)\geq \left(\frac{\delta^2}{4\dc^2}\right)^{\gamma}\right]\\&\phantom{10000}\mathbb{P}_{\numbin}\left(d^{2\gamma}(\bc_t, \partial B_j)\geq \left(\frac{\delta^2}{4\dc^2}\right)^{\gamma}\Big|\bc_t\in B_j\right)\mathbb{P}_{\numbin}\left(\bc_t\in B_j\right)\bigg)
+\alpha^{-2}r_2^2h^{2}\Exp_{\bomega,\btau,t}\left[\rho(\btau,\hat{\btau_t})\right]
\\&\geq
\frac{\alpha^{-2}}{4^{\gamma+1}}r_1^2\Lc^2\delta^{\gamma} \sum_{j=1}^{\numbin^{\dc}}\Exp_{\bomega,\btau,t}\left[\mathbb{1}\left(\hat{\omega}_{t,j}\neq\omega_j
    \right)\Big|\bc_t\in B_j,d^{2\gamma}(\bc_t, \partial B_j)\geq \left(\frac{\delta^2}{4\dc^2}\right)^{\gamma}\right]\mathbb{P}_{\numbin}\left(\bc_t\in B_j\right)
\\&\phantom{1000000000000}+\frac{1}{4}\alpha^{-2}r_2^2h^{2}\Exp_{\bomega,\btau,t}\left[\rho(\btau,\hat{\btau_t})\right].
\end{align*}
Introduce the random event $\Upsilon_{t,j}=\{\bc_t\in B_j, \ d(\bc_t,\partial B_j)\geq \frac{\delta}{2\dc}\}$. Summing both sides from $1$ to $T$, taking the maximum over $\Omega$ and $\numhypo$ and then taking the minimum over all statistics $\hat{\bomega}_{t}$ with values in $\Omega$ and $\hat{\btau}_{t}$ with values in $\numhypo$ we obtain
\begin{align*}
\min_{\hat{\bomega}_1,\dots,\hat{\bomega}_T}&\max_{\bomega\in\Omega}\sum_{t=1}^{T}\Exp_{\bomega,\btau,t}\left[\norm{\bz_t-\bx^{*}(\bomega,\bc_t)}^{2}\right]
\geq 
\frac{\alpha^{-2}}{4^{\gamma+1}}r_1^2\Lc^2 \delta^{\gamma} \numbin^{-\dc}\underbrace{\min_{\hat{\bomega}_1,\dots,\hat{\bomega}_T}\max_{\bomega\in\Omega}\sum_{t=1}^{T}\sum_{j=1}^{\numbin^{\dc}}\Exp_{\bomega,\btau,t}\left[\mathbb{1}\left(\hat{\omega}_{t,j}\neq\omega_j
\right)|\Upsilon_{t,j}\right]}_{\text{term I}}\\&\phantom{000000}+\frac{1}{4}\alpha^{-2}r_2^2h^{2}\underbrace{\min_{\hat{\btau}_1,\dots,\hat{\btau}_T}\max_{\bomega\in\Omega,\btau\in\numhypo}\sum_{t=1}^{T}\Exp_{\bomega,\btau,t}\left[\rho(\btau,\hat{\btau_t})\right]}_{\text{term II}}.
\end{align*}
We will treat terms I and II separately. For term I we can write
\begin{align*}
    \text{term I} &\geq 2^{-\numbin^{\dc} - \dx + 1}\min_{\hat{\bomega}_1,
\dots,\hat{\bomega}_T}\sum_{t=1}^{T}\sum_{j=1}^{\numbin^{\dc}}\sum_{\btau\in\numhypo}\sum_{\bomega\in\Omega}\Exp_{\bomega,\btau,t}\left[\mathbb{1}\left(\hat{\omega}_{t,j}\neq\omega_j
\right)|\Upsilon_{t,j}\right]
\\&\geq 2^{-\numbin^{\dc} - \dx + 1}\sum_{t=1}^{T}\sum_{j=1}^{\numbin^{\dc}}\sum_{\btau\in\numhypo}\sum_{\bomega\in\Omega}\min_{\hat{\bomega}_{t,j}}\Exp_{\bomega,\btau,t}\left[\mathbb{1}\left(\hat{\omega}_{t,j}\neq\omega_j
\right)|\Upsilon_{t,j}\right]
\end{align*}
Let $\Omega_j =\{\bomega\in\Omega: \omega_j = 1\}$, and for any $\bomega\in\Omega_j$, denote $\bar{\bomega}$ such that $\omega_i=\bar{\omega}_i$, for any $i\neq j$, and $\bar{\omega}_j=-1$. Thus,
\begin{align*}
    \text{term I} &\geq
     2^{-\numbin^{\dc}-\dx+1}\sum_{t=1}^{T}
     \sum_{j=1}^{\numbin^{\dc}}\sum_{
    \btau\in\numhypo
}\sum_{\bomega\in\Omega_j}\min_{\hat{\omega}_{t,j}}\left(\Exp_{\bomega,\btau,t}\left[\mathbb{1}\left(\hat{\omega}_{t,j}\neq1\right)|\Upsilon_{t,j}\right] +\Exp_{\bar{\bomega},\btau,t}\left[\mathbb{1}\left(\hat{\omega}_{t,j}\neq-1\right)|\Upsilon_{t,j}\right]\right)\\
    &\geq 
    \frac{1}{2}\sum_{t=1}^{T} \sum_{j=1}^{\numbin^{\dc}}\min_{\btau\in\numhypo,\bomega\in \Omega_j}\min_{\hat{\omega}_{t,j}}\left(\Exp_{\bomega,\btau,t}\left[\mathbb{1}\left(\hat{\omega}_{t,j}\neq1\right)|\Upsilon_{t,j}\right] +\Exp_{\bar{\bomega},\btau,t}\left[\mathbb{1}\left(\hat{\omega}_{t,j}\neq-1\right)|\Upsilon_{t,j}\right]\right).
\end{align*}
For any $\bomega\in\Omega_j$ and $\btau\in\numhypo$, let $\tilde{\mathbf{P}}_{\bomega,\btau,t}$ and $\tilde{\mathbf{P}}_{\bar{\bomega},\btau,t}$ be the conditional probability distributions of $\mathbf{P}_{\bomega,\btau,t}$ and $\mathbf{P}_{\bar{\bomega},\btau,t}$ on the event $\Upsilon_{t,j}$, respectively. Then, by Lemma \ref{lem:conKL} we have
\begin{align*}
KL\bigg(\tilde{\mathbf{P}}_{\bomega,\btau,t},\tilde{\mathbf{P}}_{\bar{\bomega},\btau,t}\bigg) &\leq I_0\sum_{i=1}^{t-1}\int\max_{\bx\in\com
}|f_{\bomega,\btau}(\bx,\bc_i) - f_{\bar{\bomega},\btau}(\bx,\bc_i)|^{2}\d \mathbb{P}_{\numbin}(\bc_i)\\&\phantom{0000}+I_0\int\max_{\bx\in\com
}|f_{\bomega,\btau}(\bx,\bc_t) - f_{\bar{\bomega},\btau}(\bx,\bc_t)|^{2}\mathbb{1}(\bc_t\in\Upsilon_{t,j})\mathbb{P}_{\numbin}^{-1}(\Upsilon_{t,j})\d \mathbb{P}_{\numbin}(\bc_t)
\\&\leq I_0\sum_{i=1}^{t-1}\int|2r_1\Lc\eta(1)d^{\gamma}(\bc_i, \partial B_j)|^2\mathbb{1}\left(\bc_i\in B_j\right)\d \mathbb{P}_{\numbin}(\bc_i) \\&\phantom{0000}+ I_0\int |2r_1\eta(1)d^{\gamma}(\bc_t, \partial B_j)|^2\mathbb{1}(\bc_t\in\Upsilon_{t,j})\mathbb{P}_{\numbin}^{-1}(\Upsilon_{t,j})\d \mathbb{P}_{\numbin}(\bc_t)
\\&\leq 4I_0r_1^2\Lc^2\eta^{2}(1)\delta^{2\gamma}\left(\sum_{i=1}^{t-1}\int\mathbb{1}\left(\bc_i\in B_j\right)\d \mathbb{P}_{\numbin}(\bc_i)+1\right)\leq4I_0r_1^2\Lc^2\eta^{2}(1)\delta^{2\gamma}\left(t\numbin^{-\dc}+1\right).
\end{align*}
Since $\delta = 1/(2\numbin)$, $\numbin = \max\left(1,\floor{\left(\min\left(1,\Lc^2\right)T\right)^{
\frac{1}{\dc+2\gamma}}}\right)$ we obtain that 
\begin{align*}
KL\left(\tilde{\mathbf{P}}_{\bomega,\btau,t},\tilde{\mathbf{P}}_{\bar{\bomega},\btau,t}\right) \leq \log(2) 
\end{align*}
if we choose $r_1\leq \left(\log(2)/(8I_0\eta^{2}(1)\max(1,\Lc^2))\right)^{1/2}$.
Therefore, using Theorem 2.12 from \cite{Tsybakov09} we get
\begin{align*}
    \text{term I} \geq \frac{T\numbin^{\dc}}{4}\exp(-\log(2)) = \frac{T\numbin^{\dc}}{8}.
\end{align*}
Moreover, for any $\bomega\in\Omega$ and $\btau,\btau'\in\numhypo$ such that $\rho(\btau,\btau')=1$ we have
\begin{align*}
KL\left(\mathbf{P}_{\bomega,\btau,t},\mathbf{P}_{\bomega,\btau',t}\right) &= \int \log\left(\frac{\d\mathbf{P}_{\bomega,\btau,t}}{\d\mathbf{P}_{\bomega,\btau',t}}\right)\d\mathbf{P}_{\bomega,\btau,t}\\
&=\int\bigg[\log\left(\frac{\d F(y_1-f_{\bomega,\btau}(\bz_1,\bc_1))}{\d F(y_1-f_{\bomega,\btau'}(\bz_1,\bc_1))}\right)+
\\&\phantom{0000000}+\sum_{i=2}^{t}\log\left(\frac{\d F(y_i-f_{\bomega,\btau}\left(\Phi_i\left(\bz_1^{i-1},y_1^{i-1},\bc_1^i\right),\bc_i)\right)}{\d F(y_i-f_{\bomega,\btau'}\left(\Phi_i\left(\bz_1^{i-1},y_1^{i-1},\bc_1^i\right),\bc_i)\right)}\right)\bigg]
\\&\phantom{000000000}\d F\left(y_i-f_{\bomega,\btau}\left(\bz_1^{i-1},\bc_1
\right)\right)\prod_{i=2}^{t}\d F\left(y_i-f_{\bomega,\btau}\left(\Phi_i\left(\bz_1^{i-1},y_{1}^{i-1},\bc_1^i\right),\bc_i
\right)\right)
\\&\leq I_0\sum_{i=1}^{t}\max_{\bx\in\com
}|f_{\bomega,\btau}(\bx,\bc_i) - f_{\bomega,\btau'}(\bx,\bc_i)|^{2}
= 4tI_0r_2^2h^{4}\eta^2(1).
\end{align*}
Using the fact that $h \leq T^{-\frac{1}{4}}$, choosing $r_2\leq \left(\log(2)/(4I_0\eta^{2}(1))\right)^{1/2}$ and applying again Theorem 2.12 from \cite{Tsybakov09} implies
\begin{align*}
    \text{term II}\geq \frac{T(\dx-1)}{4}\exp(-\log(2)) \geq \frac{T\dx}{16}.
\end{align*}
Putting together the bounds for terms I and II we get
\begin{align}\nonumber
\min_{\hat{\bomega}_1,\hat{\btau}_1\dots,\hat{\bomega}_T,\hat{\btau}_T}\max_{\bomega\in\Omega,\tau\in\numhypo}\sum_{t=1}^{T}\Exp_{\bomega,\btau,t}&\left[\norm{\bz_t-\bx^{*}(\bomega,\btau,\bc_t)}^{2}\right]\geq\frac{\alpha^{-2}}{256}r_1^2T\delta^{\gamma}  + \frac{\alpha^{-2}}{64}r_2^2 \dx Th^{2}  
\\\label{eq:lwX}&\geq \frac{\alpha^{-2}}{2048}r_3^2\min\left(1,\Lc^{\frac{2(d+\gamma)}{d+2\gamma}}\right)T^{\frac{d+\gamma}{d+2\gamma}} + \frac{\alpha^{-2}}{64}r_2^2\min\left(T, \dx T^{\frac{1}{2}}\right),
\end{align}
where $r_3 = r_2/\min(1,\Lc)$. 
Considering the fact that $r_2$ and $r_3$ are independent of 
 $\dx,\dc,\Lc$, and $T$, \eqref{eq:lwX} combined with \eqref{eq:strongc} implies the theorem.
%Note that we can choose $\cst_1,\cst_2>0$ small enough such that $r_1 \leq \cst_1\left(\alpha/\sqrt{d}\right), r_2 \leq \cst_2\alpha$, and $\cst_1,\cst_2$ are independent of $d,\dx,\alpha$ and $T$. 
\end{proof}

Recall the definition of $\mathbb{P}_{\numbin}$ in \eqref{eq:badcontexbehave}. The following figure illustrates the support of $\mathbb{P}_{\numbin}$ for $\numbin=3$ and $\dc = 2$. By definition, $\mathbb{P}_{\numbin}$ yields a uniform distribution for each black square in the figure.
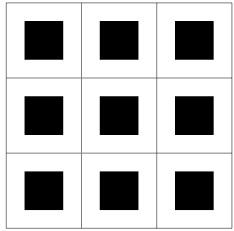
\begin{figure}[htbp]
  \centering
  \begin{tikzpicture}
    % Grid lines
    \draw[step=1cm,gray,very thin] (0,0) grid (3,3);
    % Rectangles in each cell
    \foreach \x in {0,...,2} {
      \foreach \y in {0,...,2} {
        \filldraw[black] (\x+0.25,\y+0.25) rectangle (\x+0.75,\y+0.75);
      }
    }
  \end{tikzpicture}
  \captionsetup{labelformat=empty}
  \caption{Support of $\mathbb{P}_{\numbin}$ for $\numbin=3$ and $\dc = 2$}
\end{figure}

The next lemma provides an upper bound on the Kullback-Leibler divergence of two probability measures, which plays a key role in our analysis of the lower bound.

\begin{lemma}\label{lem:conKL}
    With the same notation as in the proof of Theorem \ref{thm:lw_STRc_cont}, for $\bomega,\bomega'\in\Omega$ and $t\geq 1$ denote $\mathbf{P}_{\bomega,t}$ and $\mathbf{P}_{\bomega',t}$ as they are defined in \eqref{lw:dist}. Let $\Upsilon\subseteq [0,1]^{\dc}$ such that $\mathbb{P}_{\numbin}(\Upsilon)\neq 0$. Let $\tilde{\mathbf{P}}_{\bomega,t}$ $\tilde{\mathbf{P}}_{\bomega',t}$ be the conditional probability measures of ${\mathbf{P}}_{\bomega,t}$ ${\mathbf{P}}_{\bomega',t}$ given the event $\bc_t \in \Upsilon$, respectively. Then, if \eqref{con:lw} holds, we have
\begin{align*}
KL(\tilde{\mathbf{P}}_{\bomega,t},\tilde{\mathbf{P}}_{\bomega',t}) &\leq I_0\sum_{i=1}^{t-1}\max_{\bx\in\com}|f_{\bomega,t}(\bx,\bc_i)-f_{\bomega',t}(\bx,\bc_i)|\d \mathbb{P}_{\numbin}(\bc_i) \\&\phantom{0000}+I_o\int \max_{\bx\in\com}|f_{\bomega,t}(\bx,\bc_i)-f_{\bomega',t}(\bx,\bc_i)|\mathbb{1}(\bc_t\in\Upsilon)\mathbb{P}_{\numbin}^{-1}(\Upsilon)\d \mathbb{P}_{\numbin}(\bc_t).
\end{align*}
\end{lemma}

\begin{proof}
    By Lemma \ref{lem:con}(i), for any $\bomega\in\Omega$ we have
    \begin{align*}
    \d\tilde{\mathbf{P}}_{\bomega,t}(\bz_1,y_1^{t},\bc_1^t)&=\d F(y_1-f_{\bomega}(\bz_1,\bc_1))\d \mathbb{P}_{\numbin}(\bc_1)\left(\prod_{i=2}^{t-1}\d F(y_i-f_{\bomega}(\pi_i(\bz_1,y_1^{i-1},\bc_1^{i-1})))\d \mathbb{P}_{\numbin}(\bc_i)\right)\\&\phantom{00000}\d F(y_t-f_{\bomega}(\pi_i(\bz_1,y_1^{t-1},\bc_1^{t})))\mathbb{1}(\bc_t\in\Upsilon)\mathbb{P}_{\numbin}^{-1}(\Upsilon)\d \mathbb{P}_{\numbin}(\bc_t).
    \end{align*}
Therefore, for any $\bomega,\bomega'\in\Omega$, by Lemma \ref{lem:con}(ii) we can write
\begin{align*}
&KL\bigg(\tilde{\mathbf{P}}_{\bomega,t},\tilde{\mathbf{P}}_{\bomega',t}\bigg) = \int \log\left(\frac{\d\mathbf{P}_{\bomega,t}}{\d\mathbf{P}_{\bomega',t}}\right)\d\tilde{\mathbf{P}}_{\bomega,t}\\
&\phantom{00000000000}=\int\bigg[\log\left(\frac{\d F(y_1-f_{\bomega}(\bz_1,\bc_1))}{\d F(y_1-f_{\bomega'}(\bz_1,\bc_1))}\right)+
\\&\phantom{0000000}+\sum_{i=2}^{t}\log\left(\frac{\d F(y_i-f_{\bomega}\left(\pi_i\left(\bz_1,y_1^{i-1},\bc_1^i\right),\bc_i)\right)}{\d F(y_i-f_{\bomega'}\left(\pi_i\left(\bz_1,y_1^{i-1},\bc_1^i\right),\bc_i)\right)}\right)\bigg]
\\&\d F\left(y_1-f_{\bomega}\left(\bz_1,\bc_1
\right)\right)\d\mathbb{P}_{\numbin}(\bc_1)\left(\prod_{i=2}^{t-1}\d F\left(y_i-f_{\bomega}\left(\pi_i\left(\bz_1,y_{1}^{i-1},\bc_1^i\right),\bc_i
\right)\right)\d\mathbb{P}_{\numbin}(\bc_i)\right)\\&\phantom{00000000000000000}\d F\left(y_t-f_{\bomega}\left(\pi_t\left(\bz_1,y_{1}^{t-1},\bc_1^t\right),\bc_t
\right)\right)\mathbb{1}(\bc_t\in\Upsilon)\mathbb{P}_{\numbin}^{-1}(\Upsilon)\d \mathbb{P}_{\numbin}(\bc_t)
\\&\leq I_0\sum_{i=1}^{t-1}\int\max_{\bx\in\com
}|f_{\bomega}(\bx,\bc_i) - f_{\bomega'}(\bx,\bc_i)|^{2}\d \mathbb{P}_{\numbin}(\bc_i)\\&\phantom{00000}+I_0\int\max_{\bx\in\com
}|f_{\bomega}(\bx,\bc_t) - f_{\bomega'}(\bx,\bc_t)|^{2}\mathbb{1}(\bc_t\in\Upsilon)\mathbb{P}_{\numbin}^{-1}(\Upsilon)\d \mathbb{P}_{\numbin}(\bc_t).
\end{align*}
\end{proof}

\begin{lemma}\label{lem:family}
    With the same notation as in the proof Theorem \ref{thm:lw_STRc_cont}, let $f_{\bomega,\btau}:\mathbb{R}^{\dx}\times[0,1]^{\dc}\to\mathbb{R}$ denoted the function that it is defined in \eqref{lw:function}. Then, for all $\bomega\in\Omega$ and $\btau\in\numhypo$, $f_{\bomega,\btau}\in\mathcal{F}_{\alpha,\Lx}(M)\cap \mathcal{F}_{\gamma}(\Lc).$
\end{lemma}
\begin{proof}
Fix $\bomega\in\Omega,\btau\in\numhypo$. First, let us prove that for all $\bx\in\mathbb{R}^{\dx}$, $f_{\bomega,\btau}(\bx,\cdot)$ is a $1$-Lipschitz function. For any $\bc,\bc'\in[0,1]^{\dc}$, we can write
\begin{align*}
    |f_{\bomega,\btau}(\bx,\bc)-f_{\bomega,\btau}(\bx,\bc')| &= r_1\Lc\eta\left(x_1\delta^{-\frac{\gamma}{2}}\right)\bigg|\sum_{j=1}^{\numbin^{\dc}}\omega_j\left(d^{\gamma}(\bc,\partial B_j)\mathbb{1}\left(\bc\in B_j\right)-d^{\gamma}(\bc',\partial B_j)\mathbb{1}\left(\bc'\in B_j\right)\right)\bigg|\\&\leq r_1\Lc\eta\left(1\right)\bigg|\sum_{j=1}^{\numbin^{\dc}}\omega_j\left(d(\bc,\partial B_j)\mathbb{1}\left(\bc\in B_j\right)-d^{\gamma}(\bc',\partial B_j)\mathbb{1}\left(\bc'\in B_j\right)\right)\bigg|.
\end{align*}
For $i,j\in \numbin^{\dc}$, let us assume that $\bc\in B_i$ and $\bc'\in B_j$, and let $\by_i\in\partial B_i,\by'_j\in\partial B_j$, such that $d(\bc,\partial B_i)=\norm{\bc-\by_i}$ and $d(\bc',\partial B_j)=\norm{\bc'-\by'_j}$. If $i=j$, then
\begin{align*}
    |f_{\bomega,\btau}(\bx,\bc)-f_{\bomega,\btau}(\bx,\bc')| &\leq r_1\Lc \eta\left(1\right)\bigg|d^{\gamma}(\bc,\partial B_i)-d^{\gamma}(\bc',\partial B_i)\bigg|
\\&=r_1\Lc \eta\left(1\right)\bigg(\mathbb{1}\left(d(\bc,\partial B_i)\geq d\left(\bc',\partial B_i\right)\right)\left(\norm{\bc-\by_i}-\norm{\bc'-\by'_i}\right)\\&\phantom{00000}+\mathbb{1}\left(d(\bc,\partial B_i)< d\left(\bc',\partial B_i\right)\right)\left(\norm{\bc'-\by'_i}^{\gamma}-\norm{\bc-\by_i}^{\gamma}\right)\bigg)
\\&\leq r_1\Lc \eta\left(1\right)\bigg(\mathbb{1}\left(d(\bc,\partial B_i)\geq d\left(\bc',\partial B_i\right)\right)\left(\norm{\bc-\by'_i}^{\gamma}-\norm{\bc'-\by'_i}^{\gamma}\right)\\&\phantom{00000}+\mathbb{1}\left(d(\bc,\partial B_i)< d\left(\bc',\partial B_i\right)\right)\left(\norm{\bc'-\by_i}^{\gamma}-\norm{\bc-\by_i}^{\gamma}\right)\bigg)
\leq \Lc\norm{\bc - \bc'}^{\gamma},
\end{align*}
where the last inequality is obtained by the fact that $r_1\leq 1/(2\eta(1)\max(1,\Lc))$. If $i\neq j$, it means that $\bc$ and $\bc'$ does not belong to the same grid and we have
\begin{align*}
    d(\bc,\partial B_i)\leq \norm{\bc-\bc'}, \quad\quad\text{and}\quad\quad d(\bc',\partial B_j)\leq \norm{\bc-\bc'}.
\end{align*}
Consequently, we can write 
\begin{align*}
    |f_{\bomega,\btau}(\bx,\bc)-f_{\bomega,\btau}(\bx,\bc')| &\leq r_1\Lc\eta\left(1\right)\left(d^{\gamma}\left(\bc,\partial B_i\right)+d^{\gamma}\left(\bc',\partial B_i\right)\right)\leq \Lc\norm{\bc-\bc'}^{\gamma},
\end{align*}
where the last inequality is due to $r_1\leq 1/(2\eta(1)\max(1,\Lc)).$ 

Fixing $\bc\in[0,1]^{\dc}$, we calculate $\nabla_{\bx}^2 f_{\bomega,\btau}(\cdot,\bc)$. Let $\left(\nabla^{2}_{\bx}f_{\bomega,\btau}(\bx,\bc)\right)_{i,j}$ indicate the $(i,j)$ entry of the Hessian matrix. For $1\leq i\neq j\leq \numbin^{\dc}$ it is straightforward to check that 
$$\left(\nabla^{2}_{\bx}f_{\bomega,\btau}(\bx,\bc)\right)_{i,j}=0.$$
Now, for $i=j=1$, we have
\begin{align*}
    |\left(\nabla^{2}_{\bx}f_{\bomega,\btau}(\bx,\bc)\right)_{1,1}-2\alpha|\leq r_1\Lc\delta^{-\gamma}\max_{x\in\mathbb{R}}|\nabla^{2}_{x}\eta\left(x\right)|\sum_{j=1}^{\numbin^{\dc}}d^{\gamma}(\bc,\partial B_j)\mathbb{1}\left(\bc\in B_j\right)\leq \alpha,
\end{align*}
for $r_1\leq \alpha/\left(L'\max\left(1,\Lc\right)\right),$ where we introduced $L' = \max_{x\in\mathbb{R}}|\eta''\left(x\right)|$. Furthermore, for $i=j>1$, we can write
\begin{align*}
    |\left(\nabla^{2}_{\bx}f_{\bomega,\btau}(\bx,\bc)\right)_{i,i}-2\alpha|\leq r_2\max_{x\in\mathbb{R}}|\nabla^{2}_{x}\eta\left(x\right)|\leq \alpha,
\end{align*}
for $r_2\leq \alpha/L'$. Finally, to show that $|f(\bx,\bc)|\leq M$ for all $\bx\in\com$ and $\bc\in[0,1]^{\dc}$, we have
\begin{align*}
    |f(\bx,\bc)| \leq \alpha + r_1L\eta(1) + r_2\eta(1) \leq M.
\end{align*}
\end{proof}
\begin{lemma}\label{lemma-projections}
  For any $\bc\in B_j$ we have 
  \begin{equation}\label{eq1-lem-proj}
      d(\bc,\partial B_j) \le \frac{1}{2K},
  \end{equation}
  and, for any $\varepsilon\in (0,1)$, 
  \begin{equation}\label{eq2-lem-proj}
      \norm{\bc - \bb_j}_{\infty} \le \frac{\varepsilon}{2K} \quad \Longrightarrow \quad d(\bc,\partial B_j) \ge \frac{1-\varepsilon}{2K},
  \end{equation}
  where $\norm{\cdot}_{\infty}$ is the $\ell_\infty$ norm.
\end{lemma}
\begin{proof}
  Without loss of generality, we assume throughout the proof that $K=1$ and $\bb_j=0$, so that the closure of $B_j$ is $[-1/2,1/2]^d$ and $\partial B_j=\{\bu: \norm{\bu}_{\infty}=1/2\}$. Inequality~\eqref{eq1-lem-proj} follows from the fact that 
  $d(\bc,\partial B_j)\le d(\bc,\bc')\le {1}/{2}$, where $\bc'\in \partial B_j$ is a vector such that all its coordinates except the first one coincide with those of $\bc$, and the first coordinate $c'_1$ of $\bc'$ is equal to the projection of the first coordinate $c_1\in [-1/2,1/2]$ of $\bc$ onto the set $\{-1/2,1/2\}$.

  In order to prove \eqref{eq2-lem-proj}, we denote by $\bu^*$ a solution of the problem 
  $$
  \min_{\bu: \norm{\bu}_{\infty}=1/2}  \norm{\bc - \bu}.
  $$
  There exists $i^*$ such that $\vert u_{i^*}^*\vert = 1/2$, where $u_{i^*}^*$ is the $i^*$th coordinate of $\bu^*$. Under the assumption that  $\norm{\bc}_{\infty} \le \varepsilon/2$ we have
  $$
  d^2(\bc,\partial B_j) =\norm{\bc - \bu^*}^2 \ge \vert c_{i^*} - u_{i^*}^*\vert^2\ge \min ( \vert c_{i^*} - 1/2\vert^2, \vert c_{i^*} + 1/2\vert^2) \ge \Big(\frac{1-\varepsilon}{2}\Big)^2.
  $$
This proves \eqref{eq2-lem-proj}.
\end{proof}

\section{General comments on self-concordant barriers}\label{app:self-concord}
In \cite[Proposition 2.3.6]{nesterov1994interior}, the authors demonstrated the existence of a convex body $\com$ such that any $\mu$-self-concordant barrier over $\com$ should have $\mu\geq \dx$. The proposition is stated as follows:

\textbf{Statement of \cite[Proposition 2.3.6]{nesterov1994interior}}: Let $\com \in \mathbb{R}^{\dx}$ be a convex polytope, such that certain boundary points of $\com$ belong to $m$ of $(\dx- 1)$-dimensional facets of $\com$, where $m \in [\dx]$. Moreover, assume that the normal vectors to these facets are linearly independent. Then the value of $\mu$ for any $\mu$-self-concordant barrier on $\com$ cannot be less than $m$. In particular, the $\dx$-dimensional non-negative orthant, simplex, and hyper-cube do not admit barriers with the parameter less than $\mu$.

For the case when $\com$ is a polytope, one can construct the following self-concordant barrier \cite[Corollary 3.1.1]{nemirovski2004interior}
\begin{itemize}
    \item $\mathcal{R}(\bu) = -\sum_{j=1}^{m}\log(\langle \ba_j, \bu\rangle - b_j)$ is an $m$-self-concordant barrier for the polytope $$\com = \{\bu\in\mathbb{R}^{\dx}: \langle \ba_j, \bu\rangle \geq b_j, \quad\text{for }j\in[m] \quad\text{and}\quad\ba_j\in\mathbb{R}^{\dx}, b_j\in\mathbb{R}\}.$$
\end{itemize}
Note that in the above, if $\com$ is the $\dx$-dimensional hyper-cube then $m = 2\dx$, and in this case \cite[Proposition 2.3.6]{nesterov1994interior} indicates that the above $\mathcal{R}$ has an optimal $\mu$ as a function of $\dx$. Moreover, if $\com$ is the $\dx$-dimensional simplex then $m  = \dx + 1$, which again by \cite[Proposition 2.3.6]{nesterov1994interior} we can see that the above example of $\mathcal{R}$ has an optimal $\mu$ as a function of $\dx$. 

For the case when $\com$ is scaled version of the unit Euclidean ball with ratio $r>0$, one can consider the following self-concordant barrier \cite[Example 9.3.2]{nemirovski2004interior}
\begin{itemize}

    \item $\mathcal{R}(\bu) = -\log(r^2 - \norm{\bu}^2)$ is a $2$-self-concordant barrier for $\com$.
\end{itemize}
Considering the result in \cite[Corollary 2.3.3]{nesterov1994interior}, we have always that $\mu\geq 1$ which shows the optimality of the above $\mathcal{R}$ for the unit Euclidean ball.

We refer the reader to \cite[Chapter 9]{nemirovski2004interior} for further examples of self-concordant barrier functions. 

In the general case of a convex body $\com$, \cite{bubeck2014entropic, hildebrand2014canonical, fox2015schwarz}, proposed $\mu$-self-concordant barriers with $\mu \leq \cst\dx $, where $\cst>0$ is a non-increasing function of $\dx$. Particularly, \cite[Theorem 1]{bubeck2014entropic} shows that $\cst \leq 1 + 100\sqrt{\log(\dx)/\dx}$, for all $\dx\geq 80$. All of these constructions by \cite{bubeck2014entropic, hildebrand2014canonical, fox2015schwarz} are computationally expensive (see \cite[Section 2.2]{bubeck2014entropic} for further details on the computational aspects of these methods). Now, let us consider the special case where $\com$ is the unit $\ell_q$-ball for $q \in (2,\infty]$, i.e., 
$$\com = \{\bu\in\mathbb{R}^{\dx}: \sum_{j=1}^{\dx} |\bu_j|^q\leq 1\},$$ 
for $q\in(2,\infty)$, and 
$$\com = \{\bu\in\mathbb{R}^{\dx}: \max_{j\in[\dx]}|\bu_j|\leq 1\},$$
for $q =\infty$. In \cite[Lemma 11 and Theorem 5]{bubeck2018homotopy}, the authors showed that any $\mu$-self-concordant barrier on $\com$ should satisfy
$$\mu \geq \cst\max\left(\dx^{\frac{1}{q}}\log(\dx)^{-\frac{q}{q-2}} , \dx^{1-\frac{2}{q}}\right),$$
where $\cst > 0$ only depends on $q$. The lower bound indicates that for the case of $q = \infty$, the $(\dx + 1)$-self-concordant barrier proposed by \cite{bubeck2014entropic, hildebrand2014canonical, fox2015schwarz} is optimal with respect to $\mu$ as a function of $\dx$. However, constructing an optimal $\mu$-self-concordant barrier for any $\ell_q$-ball for $q\in(2,\infty)$ remains an open problem.

\end{document}